\PassOptionsToPackage{pdfpagelabels=false}{hyperref} 
\documentclass{article}
\usepackage[english]{babel}
\usepackage{amsmath, amsxtra, amsfonts, amssymb, amstext}
\usepackage{amsthm}
\usepackage{booktabs}
\usepackage{nicefrac}
\usepackage{xspace}
\usepackage[noadjust]{cite}
\usepackage{url}
\usepackage{graphics,color}
\usepackage[colorlinks]{hyperref}
\definecolor{linkblue}{rgb}{0.1,0.1,0.8}
\hypersetup{colorlinks=true,linkcolor=linkblue,filecolor=linkblue,urlcolor=linkblue,citecolor=linkblue}
\usepackage[algo2e,ruled,vlined,linesnumbered]{algorithm2e}
\usepackage{multirow}
\newcommand{\assign}{\leftarrow}
\usepackage{pdfpages}
\usepackage{bm}
\usepackage{enumerate}
\usepackage{pdflscape}

\usepackage{array}
\usepackage{booktabs}
\usepackage{bigdelim}

\newtheorem{theorem}{Theorem}[section]

\newtheorem{remark}[theorem]{Remark}

\newcommand{\N}{\mathbb{N}}
\newcommand{\R}{\mathbb{R}}

\newcommand{\II}{\mathbb{I}}

\renewcommand{\epsilon}{\varepsilon}
\newcommand{\eps}{\varepsilon}

\DeclareMathOperator{\E}{\mathbb{E}}

\newcommand{\Bin}{\textsc{Bin}\xspace}
\newcommand{\Poi}{\textsc{Poi}\xspace}
\newcommand{\smin}{s_{\text{min}}}

\newcommand{\ooea}{$(1 + 1)$-EA\xspace}
\newcommand{\olea}{$(1 + \lambda)$-EA\xspace}
\newcommand{\moea}{$(\mu + 1)$-EA\xspace}
\newcommand{\moga}{$(\mu + 1)$-GA\xspace}
\newcommand{\ollga}{$(1 + (\lambda,\lambda))$-GA\xspace}
\newcommand{\fooea}{$(1 + 1)$-fEA\xspace}
\newcommand{\folea}{$(1 + \lambda)$-fEA\xspace}
\newcommand{\fmoea}{$(\mu + 1)$-fEA\xspace}
\newcommand{\fmoga}{$(\mu + 1)$-fGA\xspace}

\newcommand{\onemax}{\textsc{OneMax}\xspace}
\newcommand{\OM}{\textsc{Om}\xspace}

\newcommand{\Binval}{\textsc{BinVal}\xspace}

\newcommand{\hottopic}{\textsc{HotTopic}\xspace}
\newcommand{\HT}{\textsc{HT}\xspace}
 
\begin{document}

\title{A General Dichotomy of Evolutionary Algorithms on Monotone Functions}
\author{
Johannes Lengler\\
       {ETH Z{\"u}rich}\\
       {Z\"urich, Switzerland}
}
\maketitle

\begin{abstract}
It is known that the \ooea with mutation rate $c/n$ optimises every monotone function efficiently if $c<1$, and needs exponential time on some monotone functions (\hottopic functions) if $c\geq 2.2$. We study the same question for a large variety of algorithms, particularly for \olea, \moea, \moga, their fast counterparts like fast \ooea, and for \ollga. We find that all considered mutation-based algorithms show a similar dichotomy for \hottopic functions, or even for all monotone functions. For the \ollga, this dichotomy is in the parameter $c\gamma$, which is the expected number of bit flips in an individual after mutation and crossover, neglecting selection. For the fast algorithms, the dichotomy is in $m_2/m_1$, where $m_1$ and $m_2$ are the first and second falling moment of the number of bit flips. Surprisingly, the range of efficient parameters is not affected by either population size $\mu$ nor by the offspring population size $\lambda$. 

The picture changes completely if crossover is allowed. The genetic algorithms \moga and \fmoga are efficient for arbitrary mutations strengths if $\mu$ is large enough.
\end{abstract}

\sloppy{
\section{Introduction}
\label{sec:Intro}
For evolutionary algorithms (EAs), choosing a good mutation strength is a delicate matter that is subject to conflicting goals. For example, consider a pseudo-boolean fitness function $f:\{0,1\}^n \to \R$ with standard bit mutation, i.e., all bits are flipped independently. On the one hand, if the mutation strength is too low then the progress is also slow, and the algorithm will be susceptible to local optima. On the other hand, if the mutation rate is too high and the parent is already close to a global optimum then typically the offspring, even if it has a ``good'' mutation in it, will also have a large number of detrimental mutations. A widely known example of this tradeoff are linear functions (e.g., \onemax), for which there is an optimal mutation rate $1/n$: this rate minimises the expected runtime, i.e., the expected number of function evaluations before the optimum is found. Every deviation from this mutation rate to either direction will decrease the performance. 

A different, more extreme example are (strictly) monotone pseudo-boolean functions.\footnote{We will be sloppy and drop the term ``strictly'', but throughout the paper we always mean strictly monotone functions.} A function $f:\{0,1\}^n \to \R$ is \emph{strictly monotone} if for every $x,y \in \{0,1\}^n$ with $x\neq y$ and such that $x_i \geq y_i$ for all $i \in [n]$ it holds $f(x) > f(y)$. In particular, every monotone function has a unique global optimum at $(1\ldots1)$. Moreover, every such function is efficiently optimised by \emph{random local search} (RLS), which is the $(1+1)$ algorithm that flips in each round exactly one bit, uniformly at random. From any starting point, RLS can make at most $n$ improving steps before finding the optimum, and by a coupon collector argument it will optimise any monotone function in time $O(n\log n)$. Thus, monotone function might be regarded as trivial to optimise, and we might expect every standard EA to solve them efficiently.

However, this is not so. Doerr, Jansen, Sudholt, Winzen, and Zarges showed~\cite{DoerrJSWZ10,doerr2013mutation} that even the $(1+1)$ evolutionary algorithm \ooea, which flips each bit independently with mutation rate $c/n$, may have problems. More precisely, for small mutation rate, $c<1$, the \ooea has expected runtime $O(n\log n)$ as desired, but for large mutation rate, $c >16$, there are monotone functions for which the \ooea needs exponential time. Lengler and Steger~\cite{lengler2016drift} gave a simpler construction of such ``hard'' monotone functions, which we call \hottopic, and which yield exponential runtime for $c\geq 2.13..$. The basic idea of this construction is that at every point in time there is some subset of bits which form a ``hot topic'', i.e., the algorithm considers them much more important than the other bits. An algorithm with a large mutation rate that focuses too much on the current hot topic tends to deteriorate the quality of the remaining bits. If the hot topic changes often, then the algorithm stagnates.\smallskip

Since both low and high mutation rates have their disadvantages, many different strategies have been developed to gain the best of two worlds. In this paper we pick a collection of either traditional or particularly promising methods, and analyse whether they can overcome the detrimental effect of the \hottopic functions. In particular, we consider (for constant $\mu, \lambda$) the classical \olea, \moea, and \moga, the \ollga by Doerr, Doerr, and Ebel~\cite{doerr2015black}, and the recently proposed fast \olea, fast \moea, and fast \moga~\cite{doerr2017fast}, which we abbreviate by \fooea, \folea and \fmoga, respectively. Surprisingly, for mutation-based algorithms neither $\mu$ nor $\lambda$ have any effect on the results. While we do obtain a fine-grained landscape of result (see below), one major trend is prevailing: crossover helps!

\subsection{Results}

In this section we collect our results for the different algorithms. An overview can be found in Table~\ref{tab:results}. 
Note that, unless explicitly otherwise stated, we always assume that the parameters $\mu,\lambda, c, \gamma$ of the algorithms are constant. 

\noindent \textbf{Classical EA's.} For the classical evolutionary algorithm \olea, we show a dichotomy: If the mutation parameter $c$ is sufficiently small, then the algorithms optimise all monotone functions in time $O(n \log n)$, while for large $c$ the algorithm needs exponential time on some \hottopic functions. 
The interesting question is: how does the threshold for $c$ depend on the parameters $\lambda$? It may seem that a large $\lambda$ bears some similarity with an increased mutation rate. After all, the total number of mutations in each generation is increased by a factor $\lambda$. Thus, we might expect that the \olea has difficulties with monotone functions for even smaller values of $c$. However, this is not so. The bounds on the mutation rate, $c<1$, and $c> 2.13..$ does \emph{not} depend on $\lambda$. In fact, for the \hottopic functions we can show that this is tight: If $c< 2.13..$ then the \olea and the \moga optimise all \hottopic functions in time $O(n\log n)$, while for $c>2.13..$ it is exponentially slow on some \hottopic instances.

For the \moea we get the same result on \hottopic. In particular, the threshold on $c$ is also independent of $\mu$. For the \moea, we could not show an upper runtime bound for \emph{all} monotone functions in the case $c<1$, so currently we can not exclude that the situation might get even \emph{worse} for larger $\mu$, as there might still be other monotone functions which are hard for the \moea with $c<1$.

\noindent $\bm{(\mu+1)}$\textbf{-GA.} The picture changes completely if we allow crossovers, i.e., we consider the \moga instead of the \moea. We show that for the \hottopic functions this extends the range of mutation rate arbitrarily. For every $c>0$, if $\mu$ is a sufficiently large constant then the \moga finds the optimum of \hottopic in time $O(n\log n)$. At present, there are no monotone functions known on which the \moga with arbitrary $c$ and large $\mu = \mu(c)$ is slow. It remains an intriguing open question whether the \moga with large $\mu$ is fast on \emph{every} monotone function.

\noindent $\bm{(1+(\lambda,\lambda))}$\textbf{-GA.} This algorithm creates $\lambda$ offsprings, and uses the best of them to perform $\lambda$ biased crossovers with the parent, see Section~\ref{sec:algorithms}. The best crossover offspring is then compared with the parent. This algorithm has been derived by Doerr, Doerr, and Ebel~\cite{doerr2013lessons,doerr2015black} from a theoretical understanding of so-called \emph{black-box complexity}, and has been intensively studied thereafter~\cite{doerr2015tight,doerr2015optimal,doerr2016optimal,doerr2017optimal}. Most remarkably, it gives an asymptotic improvement on the runtime of the most intensively studied test function \onemax, on which it can achieve runtime roughly $n \sqrt{\log n}$ for static settings (up to $\log\log n$ terms), and linear runtime $O(n)$ for dynamic parameter settings. These runtimes are achieved with non-constant $\lambda = \lambda(n)$. The \ollga is arguably the only known natural unbiased evolutionary algorithm that can optimise \onemax faster than $\Theta(n \log n)$. 

The algorithm comes with three parameters, the offspring population size $\lambda$, the mutation rate $c/n$ by which the offsprings are created, and a crossover bias $\gamma$ which is the probability to take the offspring's genes in the crossover. Again we find a dichotomy between weak and strong mutation, but this time not in $c$, but rather in the product $c\gamma$. In~\cite{doerr2017optimal} it is suggested to choose $c,\gamma$ in such a way that $c\gamma =1$. Note that this makes sense, because $c\gamma$ is (neglecting possible biases by the selection process) the expected number of mutations in the crossover child. Thus it is plausible that it plays a similar role as the parameter $c$ in classical algorithms. Indeed we find that for $c\gamma <1$ the runtime is small for every monotone function, while for $c\gamma > 2.13..$ it is exponential on \hottopic functions. As before, the bound is tight  for \hottopic, i.e. for $c\gamma < 2.13..$ the \ollga needs time $O(n\log n)$ to optimise \hottopic. 

Notably, the runtime benefits on \onemax carry over, at least to the \hottopic function. Since the benefits on \onemax in previous work have been achieved for non-constant parameter choices, we relax our assumption on constant parameters for the \ollga. More precisely, we show that if $\eps < c\gamma<1-\eps$ for a constant $\eps>0$, then for \emph{any} choice of $c,\gamma,\lambda$ (including non-constant and/or adaptive choices), the \ollga optimises every monotone function in $O(n\log n)$ generations. Moreover, we show that for the optimal static parameter and adaptive parameter settings in~\cite{doerr2017fast}, the algorithm achieves the same asymptotic runtime on \hottopic as on \onemax, in particular runtime $O(n)$ in the adaptive setup.\footnote{Strictly speaking, the adaptive parameter choice is not natural for \hottopic, since the parameters must be chosen as a function of the remaining zero-bits in the search points. For \hottopic, or for general monotone functions, this information is not naturally available. However, in~\cite{doerr2017optimal} it was shown that the same effect can be achieved for the \ollga by an adaptive (self-adjusting) setup using the one-fifth rule, which \emph{is} applicable for monotone functions.}

Unfortunately, it seems unlikely that the runtimes of $o(n\log n)$ for \onemax carry over to arbitrary monotone functions, because they are achieved by increasing $c$ and $\lambda$ with $n$ (although $c\gamma$ is left constant). For \onemax, if there is a zero-bit that is flipped in one of the mutations, then this mutation is always selected for crossovers. In the most relevant regime where the expected number of flipped zero-bits in \emph{any} mutation is small (say, at most one), the probability to be selected increases by a factor of $\Theta(\lambda)$ (from $1/\lambda$ to $\Theta(1)$) if a zero-bit is flipped. For monotone functions we do show that the probability of being selected can only increase with the number of flipped zero-bits. However, there is no apparent reason that it should increase by a factor of $\Theta(\lambda)$, or by any significant factor at all. In fact, it is not hard to see that for the linear function \Binval it only increases by a constant factor. 

\noindent \textbf{Fast }$\bm{(1+1)}$\textbf{-EA}, \textbf{fast } $\bm{(1+\lambda)}$\textbf{-EA}, \textbf{fast }$\bm{(\mu+1)}$\textbf{-EA.} These algorithms, which we abbreviate by \fooea, \folea, and \fmoea have recently been proposed by Doerr and Doerr~\cite{doerr2017fast}, and they have immediately attracted considerable attention~(e.g,~\cite{mironovich2017evaluation}). The idea is to replace the standard bit mutation, in which each bit is flipped independently, by a \emph{heavy-tailed} distribution $\mathcal D$. That is, in each round we draw a number $s$ from some heavy-tailed distribution (for example, a \emph{power-law distribution} with $\Pr[s = k] \sim k^{-\kappa}$ for some $\kappa >1$, also called \emph{Zipf distribution}). Then the mutation is generated from the parent by flipping exactly $s$ bits. In this way, most mutations are generated by flipping only a small number of bits, but there is a substantially increased probability to flip many bits. This approach has given some hope to unify the best of the two worlds, of small mutation rate and of large mutation rate.

For monotone functions, our results are rather discouraging. This is not completely unexpected since the algorithms build on the very idea of increasing the probability of large mutation rates. We show a dichotomy for the \fooea with respect to $m_2/m_1$, where $m_1 := \E[s]$ and $m_2 := \E[s(s-1)]$ are the first and second falling moment of the distribution $\mathcal D$, although the results are subject to some technical conditions.\footnote{Note that a heavy tail generally increases $m_2$ much stronger than $m_1$, so it increases the quotient $m_2/m_1$.} As before, if $m_2/m_1<1$ then the runtime is $O(n\log n)$ for all monotone functions. On the other hand, if $m_2/m_1\geq 2.13..$ and additionally $p_1:=\Pr[\mathcal D = 1]$ is sufficiently small then the runtime on some \hottopic instances is exponential. As for the other functions, we get a sharp threshold for the parameter regime that is efficient on \hottopic, so we can decide for each distribution whether it leads to fast or to exponential runtimes on \hottopic. Due to a correction term related to $p_1$ (Equation~\eqref{eq:hard0} on page \pageref{eq:hard0}), it \emph{is} possible to construct heavy-tail distributions which are efficient on all \hottopic functions, but they must be chosen with great care. For example, no power-law distribution with exponent $\kappa \in (1,2)$ is efficient, which includes the choice $\kappa =1.5$ that is used for experiments in~\cite{doerr2017fast} and~\cite{mironovich2017evaluation}. Also, no distribution with $p_1 < \tfrac49 \Pr[\mathcal D = 3]$ is efficient on \hottopic. In general, our findings contrast the results in~\cite{doerr2017fast}, where larger tails (smaller $\kappa$) lead to faster runtimes.

As before, larger values of $\lambda$ and $\mu$ do not seem to have any influence as long as crossover is not allowed. For the \folea and \fmoea, we show exactly the same results as for the \fooea, except that we could not show runtime bounds for \emph{all} monotone functions if $m_2/m_1<1$. Rather, we only show them for \hottopic. Thus we couldn't exclude the possibility that larger values of $\lambda,\mu$ make things even worse.
\smallskip

\noindent \textbf{Fast }$\bm{(\mu+1)}$\textbf{-GA.} As for the classical algorithms, crossover tremendously improves the situation. For every distribution $\mathcal D$ with $\Pr[\mathcal D=1] = \Omega(1)$, if $\mu$ is a sufficiently large constant then the \fmoga optimises \hottopic in time $O(n\log n)$. As for the \moga, it is an open question whether the same result carries over to \emph{all} monotone functions.\smallskip

\noindent \textbf{Further results.} For all algorithms, the regime of exponential runtime does not just mean that it is hard to find the optimum, but rather the algorithms do not even come close. More precisely, in all these cases there is an $\eps >0$ (depending only on $c$ or on the other dichotomy parameters) such that the probability that any of the EA's or GA's finds a search point with at least $(1-\eps)n$ correct bits within a subexponential time is exponentially small as $n \to \infty$. The size of $\eps$ can be quite considerable if the parameter $c$ is much larger than $2.13..$. For example, simulations suggest for the \ooea that $\eps \approx 0.15$ for $c=4$.\footnote{The parameters of the \hottopic function were $n=10,000$, $\alpha = 0.25$, $\beta = 0.05$, $\eps=0.05$, with $100$ levels, and the \ooea was run with $c= 0.9$ or $c=4$. We found that for $c=0.9$ the algorithm had optimised $99.09\% \pm 0.07$ of the bits after $100,000$ rounds and $99.98\% \pm 0.01$ after $200,000$ rounds, where the number after $\pm$ is the standard deviation. For $c=4$ the algorithm had only optimised $85.08\% \pm 0.53$ after $100,000$ rounds and this number did not visibly increase after $200,000$ rounds ($85.05\% \pm 0.40$) or $500,000$ rounds ($84.77\% \pm 0.32$). Each data point was computed from $20$ independent runs. No run with $c=4$ reached a level larger than $73$, while all runs with $c=0.9$ reached the maximum level in $100.000$ rounds.} On the other hand, starting close to the optimum does not help either: for every $\eps >0$ there are monotone function such that if the EA's or GA's are initialised with random search points with $\eps n$ incorrect bits, then still the algorithms need exponential time to find the optimum. \smallskip

\noindent \textbf{Summary.} It appears that increasing the number of offspring $\lambda$ or the population size $\mu$ does not help at all to overcome the detrimental effects of large mutation rate in evolutionary algorithms. All EA's are highly vulnerable even to a very moderate increase of the mutation rate. Using heavy tails as in the fEA's seems to make things even worse, although the picture gets more complicated. On the other hand, using crossover can remedy the effect of large mutation rates, and can extend the range of good mutation rates arbitrarily.

\subsection{Intuition on {\textsc \textbf HotTopic}}
We now give an intuition why the \hottopic functions are hard to optimise for large mutation rates. Note that a monotone function, by its very definition, has a local ``gradient'' that always points into the same corner of the hypercube, in the sense that for each bit individually, in all situations we prefer a one-bit over a zero-bit. The construction by Lengler and Steger~\cite{lengler2016drift} distorts the gradient by assigning different positive weights to the components. Such a distortion can not alter the direction of the gradient by too much. In particular, following the gradient will always decrease the distance from the optimum. This is why algorithms with small mutation rate may find the optimum; they follow the gradient relatively closely. However, the weights in~\cite{lengler2016drift} are chosen such that there is always a ``hot topic'', i.e., a subdirection of the gradient which is highly preferred over all other directions. Focusing too much on this ``hot topic'' will lead to a behaviour that is very good at optimising this particular aspect -- but all other aspects will deteriorate a little because they are out of focus. Thus if the ``hot topic'' is sufficiently narrow and changes often, then advances in this aspect will be overcompensated by decline in the neglected parts, which leads overall to stagnation. 

This last sentence is not merely a pessimistic allegory on scientific progress, but it also happens for evolutionary algorithms with large mutation rates. They will put the currently preferred direction above everything else, and will accept any mutation that makes progress in that direction, regardless of the harm that such a mutation may cause on other bits. This may lead in total to a drift away from the optimum, since random walk steps naturally tend to increase the distance from the optimum. For the fEA's or fGA's, this effect is amplified if the algorithm is close to the optimum. Then the probability to find any improvement at all is very small, and we typically find an improvement in an aggressive step in which many bits are flipped. Then the same step also typically causes a lot of errors among the low-priority bits. For the same reason, an adaptive choice of the mutation strength $c$ may be harmful if it increases the mutation parameter in phases of stagnation: close to the optimum, most steps are stagnating steps, so an adaptive algorithm might react by increasing the mutation parameter. This indeed increases the probability to find a better search point in the ``hot topic'' direction (though not the probability to make \emph{any} improvement), and may thus lead fatally to a large mutation parameter.



 \begin{landscape}
 \begin{table*}
\begin{center}
\begin{tabular}{l|l|l|l|l}
\multirow{2}{*}{\textbf{Algorithm}}    
& \textbf{$\bm{O(n\log n)}$ on} 
& \textbf{$\bm{O(n\log n)}$ on} 
& \textbf{$\bm{e^{\Omega(n)}}$ on} 
& \multirow{2}{*}{\textbf{Remarks}} \\ 
& \textbf{mon.\! funct's} 
& \textbf{\hottopic} 
& \textbf{\hottopic} 
&  \\ \hline 
\ooea 
	& $c<1$ \cite{lengler2016drift} 
	& $c < 2.13..$ 
	& $c>2.13..$\cite{lengler2016drift}\\
\olea
	& $c<1$  
	& $c < 2.13..$  
	& $c > 2.13..$  \\ 
\moea
	& ??  
	& $c < 2.13..$  
	& $c > 2.13..$  \\ 
	\hline
\moga
	& ??  
	& $c$ arbitrary$^a$ 
	& only if $\mu$ too small  
	& $^a$for $\mu = \mu(c)$ large enough\\ 
	\hline
\ollga
	& $c\gamma<1^b$  
	& $c\gamma < 2.13..^{b,c} $  
	& $c\gamma > 2.13..$ 
	& $^b$holds also if $c,\gamma, \lambda$ depend on $n$ and/or are adaptive\\ 

	&   
	&   
	&   
	& $^{c}$achieves \onemax runtimes $\approx n\sqrt{\log n}$ and $O(n)$ for  \\ 

	&   
	&   
	&   
	& \phantom{$^{c}$}optimal~\cite{doerr2015black,doerr2017optimal} static and adaptive parameters, resp. \\ 
	\hline
\fooea
	& $m_2/m_1<1$  
	& $m_2/m_1 < 1$  
	& $m_2/m_1 > 1^{d}$  
	& $^d$only if $\Pr[\mathcal D =1]$ is small enough.\\ 

	&  
	& $\Phi<1^e$  
	& $\Phi >1^e$  
	& $^e$$\Phi$ is similar to $m_2/m_1$, but has correction term for \\ 

	&   
	&  
	&   
	& \phantom{$^{e}$}$\Pr[\mathcal D =1]$, see~\eqref{eq:hard0} on page~\pageref{eq:hard0}. \\ 
	\hline

\folea
	& $m_2/m_1<1^f$  
	& $m_2/m_1 < 1$  
	& $m_2/m_1 > 1^{g}$  
	& $^f$if starting point is at most $\eps n$ from optimum.\\ 

	&  
	& $\Phi<1^h$
	& $\Phi >1$ 
	& $^g$only if $\Pr[\mathcal D =1]$ is small enough.\\

	&   
	&   
	& any power law, exp.\! $<2$  
	&  $^h$if $\Pr[\mathcal D=1] = \Omega(1)$.\\ 

	&   
	&   
	& $\Pr[\mathcal D=1] < 4/9\Pr[\mathcal D =3]$  
	&  \\ 

	\hline

\fmoea
	& ??  
	& $m_2/m_1 < 1^{i}$  
	& $m_2/m_1 > 1^{i,j}$ 
	& $^i$if $\Pr[\mathcal D =0] = \Omega(1)$.\\ 

	&  
	& $\Phi<1^i$
	& $\Phi >1^i$ 
	& $^j$only if $\Pr[\mathcal D =1]$ is small enough.\\

	&   
	&   
	& any power law, exp.\! $<2$\;$^i$  
	& \\ 

	&   
	&   
	& $\Pr[\mathcal D=1] < 4/9\Pr[\mathcal D =3]$\;$^i$  
	&  \\ 

	\hline
\fmoga
	& ??  
	& $\mathcal D$ arbitrary$^k$  
	& only if $\mu$ too small  
	& $^k$for $\mu= \mu(\mathcal D)$ large enough, if $\Pr[\mathcal D =0] = \Omega(1)$.\\ 
	
\end{tabular}
\end{center}
\caption{Overview over the results of this paper. Each entry gives a sufficient condition for the runtime statement of the corresponding column. If several lines are in one cell, then each line is a sufficient condition. Unless otherwise stated, $c,\lambda,\mu = \Theta(1)$. All results except for the \ooea are proven in this paper. The results of the first column are in Theorems~\ref{thm:easymain} and~\ref{thm:easyweakmain1}, the results of the next two columns are in Theorem~\ref{thm:concrete}, except that Remark c is in Theorem~\ref{thm:easymain}.}
\label{tab:results}
\end{table*}
\end{landscape}

\section{Preliminaries and Definitions}
\label{sec:definitions}

\subsection{Notation}
\label{sec:notation}

Throughout the paper we will assume that $f : \{0,1\}^n\to \R$ is a monotone function, i.e., for every $x,y \in \{0,1\}^n$ with $x\neq y$ and such that $x_i \geq y_i$ for all $1\leq i \leq n$ it holds $f(x) > f(y)$.\footnote{Note that this property might more correctly be called \emph{strictly monotone}, but in this paper we will stick with the shorter, slightly less precise term \emph{monotone}. In all other cases we use the standard terminology, e.g.\! the term \emph{increasing sequence} has the same meaning as \emph{non-decreasing sequence}.} We will consider algorithms that try to maximise $f$, and we will mostly focus on the \emph{runtime} of an algorithm, which we define as the number of function evaluations before the algorithm evaluates for the first time the global maximum of $f$.

We say that an EA or GA is~\emph{elitist}~\cite{doerr2017introducing} if the selection operator greedily chooses the fittest individuals to form the next generation. We call an EA or GA \emph{unbiased}~\cite{LehreW12} if the mutation and crossover algorithm are invariant under the isomorphisms of $\{0,1\}^n$, i.e., if mutation and crossover are symmetric with respect to the ordering of the bits, and with respect to exchange of the values $0$ and $1$. All algorithms considered in this paper are unbiased. 

For $n\in \N$, we denote $[n] := \{1,\ldots,n\}$. We use the notation $x = y\pm z$ to abbreviate $x \in [y-z, y+z]$. For a search point $x$, we write $\OM(x)$ for the \onemax potential, i.e., the number of one-bits in $x$. For $x\in \{0,1\}^n$ and $\emptyset \neq I \subseteq [n]$, we denote by $d(I,x) := |\{i\in I \mid x_i=0\}|/|I|$ the \emph{density} of zero bits in $I$. In particular, $d([n],x) = 1-\OM(x)/n$. 

All Landau notation $O(n), o(n), \ldots$ is with respect to $n\to \infty$. For example, $\lambda = O(1)$ means that there is a constant $C>0$, independent of $n$, such that $\lambda = \lambda(n) \leq C$ for all $n\in \N$. We say that an event $\mathcal E = \mathcal E(n)$ holds \emph{with high probability} or \emph{whp} if $\Pr[\mathcal E(n)] \to 1$ for $n\to\infty$. We say that $\mathcal E(n)$ is \emph{exponentially unlikely} if $\Pr[\mathcal E(n)] = e^{-\Omega(n)}$, and that is \emph{exponentially likely} if $\Pr[\mathcal E(n)] = 1-e^{-\Omega(n)}$. 

For an event $\mathcal E$, we denote by $\II[\mathcal E]$ the indicator variable which is one if $\mathcal E$ occurs, and zero otherwise. For a distribution $\mathcal D$, by abuse of notation write $\Pr[\mathcal D = x]$ for $\Pr[X= x \mid X \sim \mathcal D]$.

Throughout the paper, we will be slightly sloppy about conditional probabilities $\Pr[A \mid B]$ and expectation, and we will ignore cases in which $Pr[B]=0$ (e.g., in Theorem~\ref{thm:multiplicative}). We use the term \emph{increasing function} as equivalent to the term \emph{non-decreasing function}, and likewise for \emph{decreasing function}. The only exception from that pattern is for the term \emph{monotone}, where monotone functions are automatically assumed to be strictly monotone.

Finally, throughout the paper we will use $n$ for the dimension of the search space $\mu$ and $\lambda$ for the population size and offspring population size, respectively, $c$ for the mutation parameter, $\gamma$ for the crossover parameter of the \ollga, and $\mathcal D, m_1,m_2$ for the bit flip distribution of the fast EA's and GA's and its first and second moment $\E[s\mid s\sim \mathcal D]$ and $\E[s(s-1) \mid s\sim \mathcal D]$, respectively. Unless otherwise stated, we will assume that $\mu,\lambda,c,\gamma= \Theta(1)$ and $m_1 = \Omega(1)$.

\subsection{Algorithms}
\label{sec:algorithms}

\begin{algorithm2e}
 \textbf{Initialization:} \\
 \Indp
 $X \assign \emptyset$\;
 \For{$i=1,\ldots,\mu$}{
Sample $x^{(i)}$ uniformly at random from $\{0,1\}^n$\;
 $X \assign X \cup \{ x^{(i)}\}$\;
 }
 \Indm
 \textbf{Optimization:}	
 \For{$t=1,2,3,\ldots$}{
 \For{$i=1,2,\ldots \lambda$}{
 		For GA, flip a fair coin to do either a mutation or a crossover; for EA, always do a mutation. \\
                 \textbf{Mutation:} \\
		Choose $x\in X$ uniformly at random\;
		\label{line:mutation} Create $y^{(j)}$ by flipping each bit in $x$ independently with probability $c/n$\;
		\textbf{Crossover:} \\ \label{line:crossover}
		Choose $x,x'\in X$ independently uniformly at random\;
		Create $y^{(j)}$ by setting $y^{(j)}_i$ to either $x_i$ or $x_i'$, each with probability $1/2$, independently for all bits\; \label{line:endcrossover} 
	 }
\textbf{Selection:}\\
 Set $X \assign X \cup \{y^{(1)},\ldots,y^{(\lambda)}\}$\;

  \For{$i=1,\ldots, \lambda$}{
  	\label{line:selection} Select $x \in \arg\min \{f(x)\mid x\in X\}$ (break ties randomly) and update $X \assign X \setminus \{x\}$\;}

	 }
 \caption{The $(\mu+\lambda)$-EA or $(\mu+\lambda)$-GA with mutation parameter $c$ for maximizing an unknown fitness function $f:\{0,1\}^n \rightarrow \R$. The EA-algorithms skip the crossover step, line~\ref{line:crossover} to~\ref{line:endcrossover}. $X$ is a multiset, i.e., it may contain search points several times.}
\label{alg:mulambda}
\end{algorithm2e}

Most algorithms that we consider fall in the class of $(\mu+\lambda)$ evolutionary algorithms, \emph{$(\mu+\lambda)$-EAs}, or $(\mu+\lambda)$ genetic algorithms, \emph{$(\mu+\lambda)$-GAs}. They can be described by the framework in Algorithm~\ref{alg:mulambda}. In a nutshell, they maintain a population of size $\mu$. In each \emph{generation}, $\lambda$ additional offspring are created by \emph{mutation} and possibly \emph{crossover}, and the $\mu$ search points of highest fitness among the $\mu+\lambda$ individuals form the next generation. Thus we use an \emph{elitist selection} scheme. In EAs, the offspring are only created by \emph{mutation}, in GAs they are either created by mutation or by crossover. For mutation we use standard bit mutation as a default, in which each bit is independently flipped with probability $c/n$, where $c$ is the \emph{mutation parameter}. The only exception are the \emph{fast} EAs and GAs, in which first the number $s$ of bit mutations is drawn from some distribution $\mathcal D = \mathcal D(n)$, and then exactly $s$ bits are flipped, chosen uniformly at random. Recall that we will denote by $m_1 := \E[s]$ and $m_2 := \E[s(s-1)]$ the first and second falling moment of $\mathcal D$, respectively. We will always assume that $\mu, \lambda, c= \Theta(1)$.\footnote{There are many variants of the algorithms that we use here. For example, for GAs it is not important that the probability for crossover is exactly $1/2$. In fact, it is also common in GAs to create each offspring by a crossover \emph{and} a mutation. The results of the paper carry over to these variants.}

An exception to the above scheme is the \ollga~\cite{doerr2013lessons}. Here the population consists of a single search point $x$. Then in each round, we pick $s \sim \Bin(n,c/n)$, and create $\lambda$ offspring from $x$ by flipping exactly $s$ bits in $x$ uniformly at random. Then we select the fittest offspring $y$ among them, and we perform $\lambda$ independent biased crossover between $x$ and $y$, where for each bit we take the parent gene from $y$ with probability $\gamma$, and the gene from $x$ otherwise. If the best of these crossover offspring is at least as fit as $x$, then it replaces $x$. We will usually assume that $\lambda, c, \gamma = \Theta(1)$, unless otherwise mentioned.

\subsection{Hard Monotone Functions: HotTopic}
In this section we give the construction of hard monotone functions by Lengler and Steger~\cite{lengler2016drift}, following closely their exposition. The functions come with four parameters $\alpha,\beta,\rho, \eps$, and they are given by a randomised construction. We call the corresponding function $\hottopic_{\alpha,\beta,\rho,\eps} = \HT_{\alpha,\beta,\rho,\eps} = \HT$. The hard regime of parameters is
\begin{equation}\label{eq:constantsmonotone}
1 > \alpha \gg \eps \gg \beta \gg  \rho >0,
\end{equation}
by which we mean that $\alpha \in (0,1)$ is a constant, $\eps = \eps(\alpha)$ is a sufficiently small constant, $\beta = \beta(\alpha,\eps)$ is a sufficiently small constant, and $\rho = \rho(\alpha,\eps,\beta)$ is a sufficiently small constant.

Now we come to the construction. For $1 \leq i \leq e^{\rho n}$ we choose sets $A_i \subseteq [n]$ of size $\alpha n$ independently and uniformly at random, and we choose subsets $B_i\subseteq A_i$ of size $\beta n$ uniformly at random. We define the {\em level} $\ell(x)$ of a search point $x\in\{0,1\}^n$ by 
\begin{equation}\label{eq:level}
\ell(x) := \max \{ \ell' \in [e^{\rho n}] : |\{j \in B_{\ell'} : x_j = 0\} |\le \eps \beta n\},
\end{equation}
where we set $\ell(x)=0$, if no such $\ell'$ exists). Then we define $f: \{0,1\}^n \to \R$ as follows:
\begin{equation}\label{eq:hottopic}
\HT(x) :=\ell(x) \cdot n^{2} +  \sum_{i\in A_{\ell(x)+1}}x_i\cdot n +  \sum_{i\not\in A_{\ell(x)+1}} x_i, 
\end{equation}
where for $\ell = e^{\rho n}$ we set $A_{\ell+1} := B_{\ell+1} := \emptyset$.

So the set $A_{\ell+1}$ defines the ``hot topic'' while the algorithm is at level $\ell$, where the level is determined by the sets $B_i$. 
It was shown in~\cite{lengler2016drift} that whp\footnote{with high probability, i.e. with probability tending to one as $n\to\infty$.} 
the (1+1)-EA with mutation parameter $c\geq 2.2$ needs exponential time to find the optimum.


\subsection{Tools}\label{sec:tools}
We will make frequent use of the following two well-known drift theorems. The first one is the multiplicative drift theorem~\cite{DoerrJW12}.
\begin{theorem}[Multiplicative Drift]
\label{thm:multiplicative}
Let $S \subseteq \R^+$ be a finite set with minimum $\smin >0$. Let $\{X^{(t)}\}_{t\in \N}$ be a sequence of random variables over $S \cup \{0\}$. Let $T$ be the random variable that denotes the first point in time $t\in\N$ for which $X^{(t)} =0$. Suppose that there exists a constant $\delta>0$ such that
\[
\E[X^{(t)} - X^{(t+1)} \mid X^{(t)} =s] \geq \delta s 
\]
holds for all $s \in S$. Then for all $s_0 \in S$,
\[
\E[T\mid  X^{(0)} =s_0] \leq \frac{1+\ln(s_0/\smin)}{\delta}.
\]
Moreover, for all $t \geq 0$,
\begin{align*}
\Pr\left[T >  \left\lceil\frac{t+\ln(s_0/\smin)}{\delta}\right\rceil\right] \leq e^{-t} .
\end{align*}
\end{theorem}

The next theorem combines tail bounds for positive additive drift and for negative drift~\cite{Oli-Wit:j:11:negativeDrift,oliveto2012erratum,rowe2014choice,lengler2016drift}. The formulation follows~\cite{lengler2016drift}.
\begin{theorem}[Tail Bounds and Negative Drift]\label{thm:tailbounds}
For all $a,b,\delta, \xi,\eta  >0$, with $a<b$, and every function $r=r(n) = o(n/ \log n)$ there is $\rho>0$, $n_0 \in \N$ such that the following holds for all $n \geq n_0$. Let $(X^{(t)})_{t \in\N_0}$  be a Markov chain over some finite state space $S \subseteq \R$. Suppose that for all $t \geq 0$ the following conditions hold:
\begin{enumerate} 
\item $\E[X^{(t)}-X^{(t+1)} \mid X^{(t)} =s ] \geq \delta \quad \text{ for all $s>an$}$,
\item $\Pr[|X^{(t)}-X^{(t+1)}| \geq j \mid X^{(t)} =s] \leq r(1+\xi)^{-j}$ for all $j \in \N_0$ and all $s \in S$.
\end{enumerate}
Let $T_a := \min\{t \geq 0 : X^{(t)} \leq an\}$ and $T_b := \min\{t \geq 0 : X^{(t)} \geq bn\}$. Then 
\begin{enumerate}
\item[(a)] $\Pr[T_a \geq \frac{(1+\eta)(b-a)n}{\delta} \mid X^{(0)} \leq bn] \leq  e^{-\rho n/r}$.
\item[(b)] $\Pr[T_b \leq e^{\rho n} \mid X^{(0)} \leq an] \leq e^{-\rho n/r}$.
\end{enumerate}
\end{theorem}

To understand part (a), note that we would typically assume $X^{(t)}$ to need time $T_a \approx (b-a)n/\delta$ to decrease from $bn$ to $an$ if the drift is at least $\delta$. Thus (a) states that it is exponentially unlikely to exceed this time by a factor $(1+\gamma)$. Part (b) states that it is exponential unlikely to climb from $an$ to $bn$ against a negative drift, even if we allow an exponential number of steps.

We will repeatedly use Chebyshev's sum inequality~\cite{hardy1988inequalities}, also know as rearrangement inequality:
\begin{theorem}[Chebyshev's sum inequality]
\label{thm:rearrangement}
Let $(a_i)_{i\in [n]}$, $(b_i)_{i\in [n]}$ be sequences in $\R$, and let $(c_i)_{i\in [n]}$ be a sequences in $\R_0^{+}$ with $\sum_{i=1}^{n}c _i >0$ and $\sum_{i=1}^{n}c _ib_i >0$. 
\begin{enumerate}[(i)]
\item If $(a_n)$ and $(b_n)$ are both non-decreasing, then
\begin{align}\label{eq:rearrangement1}
\frac{\sum_{i=1}^n c_ia_i}{\sum_{s=1}^{n} c_i} \leq \frac{\sum_{i=1}^n c_ia_ib_i}{\sum_{s=1}^{n} c_ib_i}.
\end{align}
\item If $(a_n)$ is non-decreasing and $(b_n)$ is non-increasing, then
\begin{align}\label{eq:rearrangement2}
\frac{\sum_{i=1}^n c_ia_i}{\sum_{s=1}^{n} c_i} \geq \frac{\sum_{i=1}^n c_ia_ib_i}{\sum_{s=1}^{n} c_ib_i}.
\end{align}
\end{enumerate}
The theorem also holds for infinite sequences if all sums converge.
\end{theorem}

\section{Upper Bounds for General Monotone Functions}\label{sec:easystrong}

In this section, we will give a generic proof for strong dichotomies, i.e., for showing that under certain circumstances an algorithm will optimise every monotone function in time $O(n\log n)$. The proof follows loosely the proofs given in~\cite{doerr2013mutation} and~\cite{lengler2016drift}. 

\begin{theorem}[Generic Easyness Proof]\label{thm:easy}
Consider an elitist algorithm $\mathcal{A}$ with population size one that in each round generates an offspring by an arbitrary method, and replaces the parent if and only if the offspring has at least the same fitness. Let $s_{01}$ denote the number of zero-bits in the parent that are one-bits in the offspring, and vice versa for $s_{10}$. Assume that there is a constant $\delta>0$ such that for all $x \in \{0,1\}^n$,
\begin{align}\label{eq:easy1}
\E[s_{10} \mid \text{parent} =x \text{ and } s_{01}>0] \leq 1-\delta,
\end{align}
and
\begin{align}\label{eq:easy2}
\Pr[s_{01} > 0 \mid \text{parent} =x] = \Omega(\tfrac{1}{n}(n-\OM(x))).
\end{align}
Then with high probability $\mathcal{A}$ finds the optimum of every strictly monotone functions in $O(n\log n)$ rounds.
\end{theorem}
Before we prove the theorem, we remark that the \olea, the \fooea, and the \ollga all fit the generic description in Theorem~\ref{thm:easy}, modulo Condition~\eqref{eq:easy1}. For the \ollga, note that the procedure to generate the offspring is rather complicated, and involves several intermediate mutation and crossover steps. Nevertheless, the procedure ultimately produces a single offspring (the fittest of the crossover offsprings) which competes with the parent.

\begin{proof}[Proof of Theorem~\ref{thm:easy}]
Let $X_{t} := n-\OM(x^{(t)})$, where $x^{(t)}$ is the $t$-th search point of $\mathcal{A}$, and let $y$ be the offspring of $x^{(t)}$. First note that if $s_{01} =0$ and $x \neq y$, then by monotonicity $f(x) > f(y)$. Therefore, $x^{(t+1)} = x^{(t)}$ and $X_{t+1} = X_t$ if $s_{01}=0$. (This also holds in the trivial case $s_{01} =0$ and $x = y$). 

If $s_{01}>0$ and $s_{10} = 0$, then again by monotonicity $f(y) > f(x)$. Thus $x^{(t+1)} = y^{(t)}$ and $X_{t+1} \leq X_t -1 = X_t-1+s_{10}$.

Finally, if $s_{01}>0$ and $s_{10} > 0$ then we have two cases. Either $y^{(t)}$ is accepted, in which case $X_{t+1} = X_t -s_{01}+s_{10}\leq X_t -1 +s_{10}$. Or $y^{(t)}$ is rejected, in which case the same inequality follows from $X_{t+1} = X_t \leq X_t -1 +s_{10}$.

Summarising, we see that $X_t$ does not change for $s_{01} =0$, and that for $s_{01} >0$ we have in all cases $X_{t+1} \leq X_t -1 +s_{10}$. Therefore, $X_t$ has a drift of at least
\begin{align*}
\E[X_{t}-X_{t+1} \mid x^{(t)}] & \geq \Pr[s_{01} >0] \cdot \E[1-s_{10} \mid s_{01} >0, x^{(t)}] \\
& \stackrel{\eqref{eq:easy1},\eqref{eq:easy2}}{=} \Omega(\delta /n \cdot X_t).
\end{align*}
The claim thus follows from the multiplicative drift theorem.
\end{proof}

From Theorem~\ref{thm:easy} it will follow that the \olea with $c<1$, the \fooea with $m_2/m_1<1$, and the \ollga with $c\gamma <1$ have runtime $O(n\log n)$, since we will show that these settings satisfy~\eqref{eq:easy1}. For the \ollga with $c\gamma <1$ and non-constant parameters we cannot apply Theorem~\ref{thm:easy} directly. However, we will see that the conditional expectation in~\eqref{eq:easy1} is still the crucial object to study. 
\begin{theorem}\label{thm:easymain}
Let $\delta >0$. The following algorithms need with high probability $O(n\log n)$ generations on any strictly monotone function.
\begin{itemize}
\item The \olea with $c\leq 1-\delta$, $c=\Omega(1)$ and $\lambda = O(1)$;
\item the \fooea with $m_2/m_1 \leq 1-\delta$ and $m_1 = \Omega(1)$;
\item the \ollga with $c \gamma \leq 1-\delta$ and $c\gamma = \Omega(1)$.
\end{itemize}
Moreover, if the \ollga with $c\gamma <1-\delta$ uses the optimal static or adaptive parameter choice from~\cite{doerr2017optimal}\footnote{In fact, the suggested parameter choice in~\cite{doerr2015black,doerr2017optimal} satisfies $c\gamma =1$ instead of $c\gamma <1$. However, the runtime analysis in~\cite{doerr2015black} only changes by constant factors if $\gamma$ is decreased by a constant factor. Thus Theorem~\ref{thm:easymain} applies to the parameter choices from~\cite{doerr2015black,doerr2017optimal}, except that $\gamma$ is decreased by a constant factor.}, then with high probability the runtime on \hottopic is up to a factor $\Theta(1)$ the same as the runtime for \onemax.
\end{theorem}
We remark that the optimal runtime of the \ollga on \onemax is $O(n \sqrt{\log(n)\log\log\log(n)/\log\log n})$ for static parameters, and $O(n)$ for adaptive parameter choices~\cite{doerr2015black,doerr2017optimal}.\begin{proof}[Proof of Theorem~\ref{thm:easymain}]

First consider the \olea. 
Assume that the current search point is $x$. We create the $\lambda$ offsprings by two consecutive steps. For each $j \in [\lambda]$, first we flip every zero-bit in $x$ independently with probability $c/n$, and call the result $z^{(j)}$. Then for every one-bit in $x$, we flip the corresponding bit in $z^{(j)}$ independently with probability $c/n$, and call the result $y^{(j)}$. Thus $y^{(j)}$ follows exactly the right distribution: each bit has been flipped independently with probability $c/n$. 
Let $k \in [\lambda]$ be the random variable that denotes the index of the fittest of the $y^{(j)}$ (where we break ties randomly). Moreover, fix some index $i \in [n]$ for which $x_i =1$. 

Note that for every fixed $j$, we have $\Pr[y^{(j)}_i = 0] = c/n$. Intuitively, we need to show that the probability does not increase by the selection process. For all $j\in [\lambda]$, let $p_j := \Pr[k=j \mid z^{(1)},\ldots,z^{(\lambda)}]$. By monotonicity, if a search point $y^{(j)}$ with $y^{(j)}_i =0$ is the fittest of the offspring, then replacing $y^{(j)}_i =0$ by $y^{(j)}_i =1$ can only increase the fitness.  Therefore, conditioning on $y^{(j)}_i =0$ can only decrease the probability that $k=j$, in formula 
\begin{align}\label{eq:easyproof1}
\Pr[k=j \text{ and } y^{(j)}_i =0 \mid z^{(1)},\ldots,z^{(\lambda)}] \leq p_j \cdot \Pr[y^{(j)}_i =0],
\end{align}
where we note that the latter probability is independent of $z^{(1)},\ldots,z^{(\lambda)}$. Let $s^{(j)}_{01}$ be the number of zero-bits in which $x$ and $z^{(j)}$ differ, i.e., the number of zero-bits that have been flipped into one-bits. Let $J := \{j\in[\lambda] \mid s^{(j)}_{01} > 0\}$. Then by~\eqref{eq:easyproof1},
\begin{align*}
\Pr[y^{(k)}_i = 0 \mid  s^{(k)}_{01} > 0; z^{(1)},\ldots,z^{(\lambda)}] & \leq \frac{\sum_{j\in J} p_j \cdot \Pr[y^{(j)}_i =0]}{\sum_{j\in J} p_j} \\ 
&\leq c/n.
\end{align*}
Summing over all $i \in [n]$ with $x_i=1$, and averaging over all possible values of $z^{(1)},\ldots,z^{(\lambda)}$, we obtain
\begin{align}\label{eq:easyproof1b}
\E[|\{i\in [n] : x_i =1, y^{(k)}_i = 0\}| \mid s^{(k)}_{01} > 0] \leq c \leq 1-\delta.
\end{align}
Thus Condition~\eqref{eq:easy1} in Theorem~\ref{thm:easy} is satisfied. Note that so far we have not used $c = \Omega(1)$ and $\lambda =O(1)$. We only need these assumptions to verify Condition~\eqref{eq:easy2}, which indeed follows immediately. So the statement for the \olea follows from Theorem~\ref{thm:easy}. \medskip

Next we turn to the \fooea. As before, let $x$ be the current search point, and let $y$ be the offspring. For all $s> 0$, let $p_s$ be the probability to flip exactly $s$ bits. 
Then we have 
\begin{align}\label{eq:easyproof2b}
2\cdot\Pr[1 \leq s \leq 2] & \geq m_1 - \sum_{s=3}^{\infty} p_s s  \geq m_1 - \frac{1}{2}\sum_{s=3}^{\infty} p_s s(s-1) \nonumber \\
& \geq m_1 -\frac{m_2}{2} \geq \frac{m_1}{2},
\end{align}
so $\Pr[1 \leq s \leq 2] \geq m_1/4$. In particular,this expression is in $\Omega(1)$, which implies Condition~\eqref{eq:easy2} in Theorem~\ref{thm:easy}. Note for later reference that~\eqref{eq:easyproof2b} also implies $m_1 \leq 4$. 

It remains to check~\eqref{eq:easy1}. For this, let $s_{01}$ and $s_{10}$ denote the number of bit flips from $0$ to $1$ and from from $1$ to $0$ respectively. We observe that
\begin{align}
\E[s_{10} \mid s_{01}>0] & = \frac{\sum_{s\geq 1} p_s \Pr[s_{01} >0 \mid s] \cdot \E[s_{10} \mid s; s_{01}>0]}{\sum_{s\geq 1} p_s \Pr[s_{01} >0 \mid s]} \nonumber \\
& \leq \frac{\sum_{s\geq 1} p_s \Pr[s_{01} >0 \mid s] \cdot (s-1)}{\sum_{s\geq 1} p_s \Pr[s_{01} >0 \mid s]}. \label{eq:easyproof2}
\end{align}
To estimate the term above, we note that since the term $s-1$ is increasing, for every non-decreasing sequence $\alpha_s$, by Chebyshev's sum inequality we may bound
\begin{align}\label{eq:easyproof3}
\eqref{eq:easyproof2} \leq \frac{\sum_{s\geq 1} p_s \Pr[s_{01} >0 \mid s] \cdot \alpha_s \cdot (s-1)}{\sum_{s\geq 1} p_s \Pr[s_{01} >0 \mid s]\cdot \alpha_s}.
\end{align}
We will use $\alpha_s := s/\Pr[s_{01} >0 \mid s]$, so we need to show that $\alpha_s^{-1} = \Pr[s_{01} >0 \mid s]/s$ is a non-increasing sequence. We regard the process where we draw the $s$ bit positions one after another, and consider for the $i$-th round the probability $q_i$ that a zero bit is drawn for the first time in this round. This probability is decreasing, and thus 
\begin{align*}
\frac1s \Pr[s_{01} >0 \mid s]  = \frac1s \sum_{i=1}^s q_i &\leq \frac{1}{s-1} \sum_{i=1}^{s-1} q_i \\ 
&= \frac{1}{s-1} \Pr[s_{01} >0 \mid s-1],
\end{align*}
as desired. Plugging $\alpha_s$ into~\eqref{eq:easyproof3} yields
\begin{align}\label{eq:easyproof3b}
\E[s_{10} \mid s_{01}>0] \leq \frac{m_2}{m_1} \leq 1-\delta,
\end{align}
so Condition~\eqref{eq:easy1} in Theorem~\ref{thm:easy} is satisfied, and the statement follows from Theorem~\ref{thm:easy}. 
For later reference, we note that the first inequality in~\eqref{eq:easyproof3b} holds for any distribution $\mathcal D$, regardless whether $\frac{m_2}{m_1} \leq 1-\delta$.
\medskip

Finally let us turn to the \ollga. 
Let $x$ be the current search point. In the first step an integer $s \sim \Bin(n,c/n)$ is chosen, and $\lambda$ offspring $y^{(1)},\ldots,y^{(\lambda)}$ are created from $x$ by flipping exactly $s$ bits. As for the \olea, let $s_{01}^{(j)}$ and $s_{10}^{(j)} = s-s_{01}^{(j)}$ be the number of zero-bits and one-bits that were flipped in the creation of $y^{(j)}$, respectively, and let $k \in [\lambda]$ be the fittest among the $y^{(j)}$, breaking ties randomly. Note that for a fixed $j$, the offspring $y^{(j)}$ has the same distribution as for the \olea. (The difference is that the offspring are not independent.) Therefore for every fixed $j\in [\lambda]$ and all $r,r'\in \N$,
\begin{align}\label{eq:easyproof4}
\Pr[s_{10}^{(j)} \geq r \mid s_{01}^{(j)} \geq r'] = \Pr[\Bin(\OM(x),c/n) \geq r].
\end{align}
In particular, $\E[s_{10}^{(j)} \geq r \mid s_{01}^{(j)} \geq r'] \leq c$.

Now we show that $\E[s_{10}^{(j)} \mid s_{01}^{(j)} \geq r']$ can only increase by the selection process. Fix any values of $s, s_{01}^{(1)},\ldots,s_{01}^{(\lambda)}$, and let $p_j := \Pr[k=j \mid s,s_{01}^{(1)},\ldots,s_{01}^{(\lambda)}]$. Thus we condition on the number of zero-bits and one-bits that we flip, but not on their location. Note that we can create a random offspring with $s_{01}^{(j)} = \sigma$ (i.e., with $\sigma$ flips of one-bits and $s-\sigma$ flips of zero-bits of $x$) by starting with a random search point with $s_{01}^{(j)} = \sigma-1$, reverting a random flip of a zero-bit, and adding a random flip of a one-bit of $x$. Since this operation strictly increases the fitness, it can only increase $p_j$. Hence $p_j$ is an increasing function in $s_{01}^{(j)}$. On the other hand, the indicator random variable $\II[s_{01}^{(j)}\leq s-r]$ is trivially decreasing in $s_{01}^{(j)}$. Therefore, by Chebyshev's sum inequality, and using $J:= \{j\in \lambda \mid s_{01}^{(j)}\geq r'\}$ we have for any $r\in \N$,
\begin{align*}
\Pr&[s_{10}^{(k)}\geq r \mid s,s_{01}^{(1)},\ldots,s_{01}^{(\lambda)} \text{ and } s_{01}^{(k)}\geq r']  \\
&  = \frac{\sum_{j\in J} p_j \cdot \II[s_{10}^{(j)}\geq r]}{\sum_{j\in J}p_j} = \frac{\sum_{j\in J} p_j \cdot \II[s_{01}^{(j)} \leq s-r]}{\sum_{j\in J}p_j} \\
& \leq \frac{1}{|J|} \sum_{j\in J}\II[s_{01}^{(j)}\leq s-r] = \frac{1}{|J|} \sum_{j\in J}\II[s_{10}^{(j)} \geq r].
\end{align*}
Note that the latter term just counts which fraction of those $j$ with $s_{01}^{(j)} \geq r'$ also satisfy $s_{10}^{(j)}\geq r$. This is directly related to the definition of conditional probability. In particular, averaging over all possible values of $s,s_{01}^{(1)},\ldots,s_{01}^{(\lambda)}$, we get for every fixed $j\in [\lambda]$, 
\begin{align}\label{eq:easyproof5}
\Pr[s_{10}^{(k)} \geq r \mid s_{01}^{(k)}\geq r']  \leq \Pr[s_{10}^{(j)}\geq r \mid s_{01}^{(j)} \geq r'].
\end{align}
In other words, $s_{10}^{(k)}$ is stochastically dominated by $s_{10}^{(j)}$ if we condition on $s_{01}\geq r'$. Recall that the latter one is a binomial distribution by~\eqref{eq:easyproof4}, and in particular $\E[s_{10}^{(k)} \mid s_{01}^{(k)}\geq r'] \leq c$. By an analogous argument, the selection process can only increase how many zero-bits of $x$ are flipped, i.e., $s_{01}^{(k)}$ stochastically dominates $s_{01}^{(j)}$ if we condition on $s_{01}\geq r'$.

In the second step of the \ollga, the algorithm produces $\lambda$ biased crossovers $z^{(1)},\ldots,z^{(\lambda)}$ between $x$ and $y^{(k)}$, choosing the bits from $y^{(k)}$ with probability $\gamma$. Then it compares the fittest crossover offspring $z^{(\ell)}$ with $x$. Similarly as before, we let $t_{01}^{(j)}$ be the number of bits that are zero in $x$ and one in $z^{(j)}$, and vice versa for $t_{10}^{(j)}$. To estimate $\E[t_{10}^{(j)} \mid t_{01}^{(j)}>0]$, we define the following three terms:
\begin{align*}
A_\sigma & := \Pr[t_{01}^{(j)} >0 \mid s_{01}^{(k)} = \sigma]/\Pr[t_{01}^{(j)} >0];\\
B_\sigma & := \Pr[s_{01}^{(k)} =\sigma];\\
C_\sigma & := \E[t_{10}^{(j)} \mid t_{01}^{(j)} >0 \text{ and }s_{01}^{(k)} = \sigma].
\end{align*}
We observe that $A_\sigma$ is increasing in $\sigma$ with $A_0=0$, and that $\sum_{\sigma =0}^{\infty} B_\sigma = \sum_{\sigma =0}^{\infty} A_\sigma B_\sigma =1$. Moreover, observe that conditioned on $s_{01}^{(k)} = \sigma$, the term $t_{10}^{(j)}$ is independent of the event $t_{01}^{(j)}>0$, since the crossover treats bits independently. Thus we may equivalently write $C_\sigma = \E[t_{10}^{(j)} \mid s_{01}^{(k)} = \sigma]$. Consequently, for every $\sigma_0\geq 0$, using~\eqref{eq:easyproof5} in the last step,
\begin{align*}
\sum_{\sigma' =\sigma}^{\infty}B_{\sigma'} C_{\sigma'} & = \Pr[s_{01}^{(k)} \geq \sigma]\E[t_{10}^{(j)} \mid s_{01}^{(k)} \geq \sigma] \\
& = \Pr[s_{01}^{(k)} \geq \sigma] \cdot \gamma \E[s_{10}^{(k)} \mid s_{01}^{(k)} \geq \sigma] \\
& \leq c\gamma \Pr[s_{01}^{(k)} \geq \sigma] = c\gamma \sum_{\sigma' =\sigma}^{\infty}B_{\sigma'}.
\end{align*}
Using summation by parts (discrete partial integration) on the two functions $g_1(\sigma) = A_\sigma$ and $g_2(\sigma) = \sum_{\sigma' = \sigma}^{\infty}B_{\gamma'}C_{\gamma'}$ (and backwards for $g_2'(\sigma) = \sum_{\sigma' = \sigma}^{\infty}B_{\gamma'}$), we thus may conclude that
\begin{align*}
\sum_{\sigma' =\sigma}^{\infty}A_{\sigma'} B_{\sigma'} C_{\sigma'} & = \sum_{\sigma'=0}^{\infty}\underbrace{(A_{\sigma'}-A_{\sigma'-1})}_{\geq 0}\underbrace{\sum_{r=\sigma'}^{\infty}B_rC_r}_{\leq c\gamma \sum B_{r}} \\
& \leq c\gamma \sum_{\sigma'=0}^{\infty}(A_{\sigma'}-A_{\sigma'-1})\sum_{r=\sigma'}^{\infty}B_r \\
& = c \gamma \sum_{\sigma' =\sigma}^{\infty}A_{\sigma'} B_{\sigma'} = c\gamma
\end{align*}
Indeed we have computed a term of interest:
\begin{align*}
\E[t_{10}^{(j)} \mid t_{01}^{(j)}>0] & = \sum_{\sigma =0}^{\infty}\Pr[s_{01}^{(k)} = \sigma \mid t_{01}^{(j)}>0] \cdot C_{\sigma} \\
& = \sum_{\sigma =0}^{\infty}A_{\sigma} B_{\sigma} C_{\sigma} \leq c\gamma < 1-\delta.
\end{align*}

It remains to show that the second selection process, picking the fittest among the $z^{(1)},\ldots,z^{(\lambda)}$, does not increase the term $\E[t_{10}^{(j)} \mid t_{01}^{(j)}>0]$, i.e.,
\begin{align*}
\E[t_{10}^{(\ell)} \mid t_{01}^{(\ell)}>0] \leq \E[t_{10}^{(j)} \mid t_{01}^{(j)}>0] \leq c\gamma \leq 1-\delta
\end{align*}
The argument is again the same as before: it suffices to observe that the probability that $z^{(j)}$ is the fittest crossover offspring is decreasing in $t_{01}^{(j)}$, and the claim follows from Chebyshev's sum inequality. We skip the details. This proves Condition~\eqref{eq:easy1} in Theorem~\ref{thm:easy}. For Condition~\eqref{eq:easy2} we fix any $j\in [\lambda]$, and note that $\E[s_{01}^{(j)}] = \tfrac{c}{n}(n-\OM(x))$.
Since the probability to select $y^{(j)}$ is increasing in $s_{01}^{(j)}$, we have $\E[s_{01}^{(k)}] \geq \E[s_{01}^{(j)}] = \tfrac{c}{n}(n-\OM(x))$. In particular, each crossover mutation satisfies
\[
\E[t_{01}^{(j)}] = \gamma \E[s_{01}^{(k)}] \geq \frac{c\gamma}{n}(n-\OM(x)).
\]
As before, the probability to select $z^{(j)}$ in the second selection step is increasing in $t_{01}^{(j)}$. Therefore, $\E[t_{01}^{(\ell)}] \geq \frac{c\gamma}{n}(n-\OM(x))$, which implies Condition~\eqref{eq:easy2} since $c\gamma = \Omega(1)$. Note that this is the only step in the proof where we use $c\gamma = \Omega(1)$. This concludes the proof of the first statement for the \ollga. \smallskip

We come to the second statement on the \ollga, on \hottopic for the optimal parameter choices in~\cite{doerr2015black,doerr2017optimal}. The crucial observation is that by the negative drift theorem the number of zero bits will drop below $\eps \beta n$ in $O(n)$ generations. In particular, in the set $B_i$ there are at most $\eps \beta n$ zero bits, which means that the level has reached its maximum. This phase needs runtime $O(\ell n)$, since $2\ell$ search points are created in each generation. Inspecting the proofs in~\cite{doerr2015black} and~\cite{doerr2017optimal}, we find that this bound asymptotically equals the total runtime of the \ollga on \onemax. (For the static parameter setting we have a runtime of $\Theta(\lambda n)$ for the optimal $\lambda = \sqrt{\log n \log \log n / \log \log \log n}$, for the adaptive setting we have $\lambda \leq \sqrt{1/(\eps \beta)} = O(1)$ in this region of the search space, so $O(\lambda n) = O(n)$.) Moreover, again by the negative drift theorem, once the number of zero-bits has dropped below, say, $\eps \beta n/4$, whp it will not increase again to more than $\eps\beta n/2$ zero-bits for $\omega(n\log n)$ rounds. So let us assume that the number of zero-bits stays below $\eps \beta n/2$. Again inspecting the proofs, we see that $\lambda = O(\sqrt{n})$ throughout the process, so whp no offspring will ever leave the range of at most $\eps \beta n$ zero-bits. However, in this range the \hottopic function is up to an additive constant equal to the \onemax function, so the remaining optimisation time for \hottopic and for \onemax coincides. This proves the theorem. 
\end{proof}

Our next theorem gives upper bounds on the runtime of the \folea on any monotone function, provided that $m_2/m_1 <1$, where $m_1$ and $m_2$ are the first and second falling moments of the flip number distribution $\mathcal D$. We need to make the assumption that the algorithm starts at most in distance $\eps n$ to the optimum. It is unclear whether this assumption is necessary, or merely an artefact of our proof.

\begin{theorem}\label{thm:easyweakmain1}
Let $\delta >0$ be a constant, let $\lambda = O(1)$, and consider the \folea with distribution $\mathcal D  = \mathcal D(n)$, whose falling moments $m_1,m_2$ satisfy $m_2/m_1 \leq 1-\delta$ and $m_1 = \Omega(1)$. Then there is $\eps>0$ such that the \folea starting with any search point with at most $\eps n$ zero-bits finds the optimum of every strictly monotone functions in time $O(n\log n)$ with high probability.
\end{theorem}
\begin{proof}
We set $\eps := \delta/(96C)$, where $C$ is some constant upper bound on $\lambda$. Let $x$ denote the current search point. Assume for now that $d(x, [n]) \leq 2\eps$, i.e., that $x$ contains at most $2\eps n$ zero-bits. We will justify this assumption at the end of the proof. Let $y^{(1)}, \ldots,y^{(\lambda)}$ be the offspring of $x$, and let $k\in [\lambda]$ be the index of the fittest offspring. For any fixed $j \in [\lambda]$, let $s_{01}^{(j)}$ and $s_{10}^{(j)}$ be the number of zero-bits and one-bits that were flipped in the creation of $y^{(j)}$, respectively. Moreover, let $J := \{j\in [\lambda] \mid s_{01}^{(j)} >0\}$.

Our first step is to bound $\E[|J| \mid J\neq \emptyset]$. Observe that $m_1 \leq 4$ by \eqref{eq:easyproof2b}. Therefore, $\E[s_{01}^{(j)}] \leq 8 \eps$ for all $j\in [\lambda]$, and $\Pr[s_{01}^{(j)} >0] \leq 8 \eps$ by Markov's inequality. Since the offspring are generated independently, $|J|$ follows a binomial distribution with expectation $\E[|J|] \leq 8\eps\lambda$, and hence $\Pr[|J| \geq r \mid J \neq \emptyset] \leq (8\eps\lambda)^{r-1}$ for all $r\in \N$. Thus, since $\eps \leq 1/(16\lambda)$,
\begin{align*}
\E[|J| \mid J \neq \emptyset] & \leq \sum_{r=1}^{\infty} r \cdot \Pr[|J| \geq r \mid J \neq \emptyset] \\
& \leq 1+ 8\eps\lambda \sum_{r=2}^{\infty} r (8\eps\lambda)^{r-2} \\
& \leq 1+ 8\eps\lambda \sum_{r=2}^{\infty} r (1/2)^{r-2} = 1+48\eps\lambda \leq 1+\frac{\delta}{2}.
\end{align*}
Now we are ready to bound $\E[s_{10}^{(k)} \mid s_{01}^{(k)}>0]$. For a fixed $j \in [\lambda]$, by~\eqref{eq:easyproof3b} we have $\E[s_{10}^{(j)} \mid s_{01}^{(j)}>0] \leq 1-\delta$. Therefore, bounding generously,
\begin{align*}
\E[s_{10}^{(k)} \mid s_{01}^{(k)}>0] &\leq \E[\sum_{j\in J} s_{10}^{(j)} \mid J \neq \emptyset] \\
& = \E[|J| \mid J \neq \emptyset]\cdot \E[s_{10}^{(j)} \mid s_{01}^{(j)}>0] \\
& \leq (1+\delta/2)\cdot (1-\delta) \leq 1-\frac{\delta}{2}.
\end{align*}
Therefore, Condition~\eqref{eq:easy1} from Theorem~\ref{thm:easy} is satisfied. As before, Condition~\eqref{eq:easy2} is easy to check. So Theorem~\ref{thm:easy} would imply the statement if we would know that no search point has more than $2\eps n$ zero-bits. So it remains to show that whp this is the the case. Let $X_t := n-\OM(x^{(t)})$ be the number of zero-bits in the $t$-th generation. Since $m_2 < m_1 \leq 4$, we know that the number $s$ of mutations satisfies $\Pr[s \geq r] \leq \tfrac{4}{r(r-1)}$ by Markov's inequality. In particular, for $r_0 := n^{3/4}$ we get $\Pr[s \geq r_0] =O(n^{-3/2})$. Thus, among the first $O(n\log n)$ rounds the number of mutations in which at least $r_0$ bits are flipped is at most $O(n^{-3/2}\cdot \lambda n\log n) = o(1)$. So by Markov's inequality, whp there will be no such rounds, and the maximal step size of $X_t$ is $r_0$. We claim that whenever $X_t \geq \eps n$, then with very high probability the random walk will fall below $\eps n$ before it reaches at least $2\eps n$. Indeed this follows from the negative drift theorem applied to the random variable $Y_t := X_t/n^{3/4}$. This random walk $Y_t$ has at most constant step size, and it has drift $-\Omega(n^{-3/4})$, since $X_t$ has constant drift in the range between $\eps n$ and $2\eps n$. Therefore, by the negative drift theorem, the probability to reach $2\eps n$ before falling back below $\eps n$ is $\exp\{-\Omega(n^{3/4})\}$. This applies to each phase at which $X_t \geq \eps n$. Since there are at most $O(n\log n)$ such phases, whp there is no phase in which $X_t$ reaches at least $2\eps n$. Hence, whp the optimum is reached before we reach a search point with more than $2\eps n$ zero-bits. This concludes the proof.
\end{proof}

\section{Generic Result for HotTopic}\label{sec:generic}
In this section we analyse the behaviour of a generic algorithm on \hottopic, which will later serve as basis for all our results on \hottopic for concrete algorithms. The generic algorithm uses population size one, but we will show that, surprisingly, $(\mu+1)$ algorithm can be described by the same framework.
%

\begin{theorem}[\hottopic, Generic Runtime]\label{thm:generic}
Let $0<\alpha <1$. Consider an elitist, unbiased optimisation algorithm $\mathcal{A}$ with population size one that starts with a random search point $x$ and in each round generates an offspring $y$ by an arbitrary (unbiased) method, and replaces the parent $x$ by $y$ if $\HT(y) > \HT(x)$. For equal fitness, it may decide arbitrarily whether it replaces the parent. Let $s$ be the random variable that denotes the total number of bits in which parent and offspring differ, and note that the distribution of $s$ may depend on the parent. For parent $x$, we define
\begin{align}\label{eq:hard0}
\Phi(x) := \frac{\E[s(s-1) (1-\alpha)^{s-1}]}{\E[s(1-\alpha)^{s-1}]} - \frac{\tfrac{(1-\alpha)}{\alpha}\Pr[s=1]}{\E[s(1-\alpha)^{s-1}]}.
\end{align}
\begin{enumerate}[(a)]
\item If there are constants $\zeta,\zeta' >0$ such that for all $x \in \{0,1\}^n$ with at most $\zeta n$ zero-bits,
\begin{align}\label{eq:hard1}
\Phi(x) \geq 1 +\zeta',
\end{align}
then with high probability $\mathcal A$ needs an exponential number of steps to find the global optimum of $\hottopic_{\alpha,\beta,\rho,\eps}$ with parameters $\beta,\rho,\eps$ as in~\eqref{eq:constantsmonotone}. 
\item If there are constants $\zeta,\zeta' >0$ such that for all $x \in \{0,1\}^n$ with at most $\zeta n$ zero-bits,
\begin{align}\label{eq:hard2}
\Phi(x) \leq 1-\zeta',
\end{align}
and if moreover $\Pr[s=1] \geq \zeta$ and $\E[s(s-1)] \leq 1/\zeta$ for all parents $x$, then with high probability $\mathcal A$ needs $O(n\log n)$ steps to find the global optimum of $\hottopic_{\alpha,\beta,\rho,\eps}$ with parameters $\beta,\rho,\eps$ as in~\eqref{eq:constantsmonotone}.
\item The statements in (a) and (b) remain true for algorithms that are only unbiased conditioned on an improving step\footnote{I.e., assume that for parent $x$, the next search point is drawn from some distribution $\mathcal X$, which is not necessarily unbiased. Then we require that there is an unbiased distribution $\mathcal X'$ such that $\Pr[y \in \mathcal X \mid \HT(y)>\HT(x)] = \Pr[y \in \mathcal X' \mid \HT(y)>\HT(x)]$.}, if in (b) we also have $\Pr[\text{improving step}] \geq \zeta\cdot d([n],x)$. Moreover, the statement in (b) remains true for algorithms that are only unbiased if $x$ has more than $\zeta n$ zero-bits, and possibly biased for at most $\zeta n$ zero bits, if we replace~\eqref{eq:hard2} by the condition $\E[s \mid \HT(y) > \HT(x)] \leq 2-\zeta$.
\end{enumerate}
Finally, there is a constant $\eta = \eta(\zeta',\alpha) >0$ independent of $\zeta$ such that (a), (b), and (c) remain true in the presence of the following adversary $A$. Whenever an offspring $x'$ is created from $x$ that satisfies $f(x')> f(x)$ then $A$ flips a coin. With probability $1-\eta$, she does nothing. Otherwise, she draws an integer $\tau \in \N$ with expectation $O(1)$ and she may change up to $\tau$ bits in the current search point. For (a) we additionally require $\Pr[\tau \geq \tau'] = e^{-\Omega(\tau')}$, while for (b) and (c) we only require $\Pr[\tau \geq n^{1-\eta}] = o(1/(n\log n))$.
\end{theorem}
We remark that (b) and (c) require parameters as in~\eqref{eq:constantsmonotone}, and thus do not exclude a large runtime on \hottopic for atypical parameters, e.g., for large $\eps$.
\begin{proof}[Proof of Theorem~\ref{thm:generic}]
\emph{(a)} Recall that $d(x,S)$ denotes the density of zero bits in the index set $S$. Apart from the parameters, we will use three more constants $\tilde\eps$, $\delta$ and $\xi$ in the proof, such that
 \begin{equation}\label{eq:constantsmonotone2}
1 > \alpha \gg \tilde \eps\gg \eps \gg \delta \gg \beta \gg \xi \gg \rho >0,
\end{equation}
by which we mean by slight abuse of notation that all these parameters are constant, but that each parameter is chosen sufficiently small with respect to the preceding parameters. We will not hide any of these parameters in Landau notation except for $\alpha$; for example, $\sigma = O(\beta)$ means that there is a constant $C = C(\alpha)$ independent of $\tilde \eps,\eps,\delta,\beta,\xi,\rho$ such that $\sigma \leq C\beta$. Note that due to the ordering, $\sigma=O(\beta)$ implies $\sigma=O(\delta)$ etc.

We first sketch the main idea, which we take from~\cite{lengler2016drift}. Let $x$ be the current search point. We consider the algorithm on level $\ell = i-1$ in the case that $\eps \leq d(A_{i},x), d(R_i,x) \leq \eps +\delta$ for $R_{i} := [n]\setminus A_{i}$. We will show that in this regime the density of $R$ has a \emph{positive} drift, away from zero. Moreover, the drift is so strong that in the time that $d(A_{i},x)$ reaches $\approx \eps$ (at which point the level increases), the density $d(R_i,x)$ will have increased to at least $\eps + \delta$. Note that if this is true, when the level increases the total density in the string to at least $\alpha\eps + (1-\alpha)d(x,R_i) \geq \eps + (1-\alpha)\delta$. In particular, we should expect that the density of $A_{i+1}$ and of its complement $R_{i+1}$ is at least roughly $\eps + (1-\alpha)\delta$. Therefore, if we can estimate the negative drift $\Delta$ of $A_{i+1}$ and $\Delta'$ of $R_{i+1}$, then we know that $d(A_{i+1},x)$ will need at least $\approx (1-\alpha)\delta / |\Delta|$ rounds to drop from $\eps + (1-\alpha)\delta$ to $\eps$. On the other hand, $d(R_{i+1},x)$ will need at most $\approx \alpha\delta / \Delta'$ rounds to climb from $\eps + (1-\alpha)\delta$ to $\eps+\delta$. If the latter time is larger, then we can repeat this game on the next level. Comparing these two duration yields~\eqref{eq:hard0}.\smallskip


Before we compute the drift, let us fix some notation and discuss in more detail why the argument works. We will denote the $t$-th search point by $x^{(t)}$ and its offspring by $y^{(t)}$, but we will abbreviate $\ell(t) := \ell(x^{(t)})$ and $d(S,t) := d(S,x^{(t)})$ for $S\subseteq [n]$. We denote by $s^{(t)}$ the number of bit flips in the $t$-th step, and by $s_{01}^{(t)}$ and $s_{10}^{(t)}$ the number of bit flips from zero to one or from one to zero, respectively. For a set $S \subseteq [n]$ of indices, we denote $s^{(t,S)}$, $s_{01}^{(t,S)}$ and $s_{10}^{(t,S)}$ the corresponding number of bits in $S$. In all notation, we drop $t$ if it is clear from the context. Let $T := e^{\rho n}$. Then $T$ is the number of levels, and we will show that whp the algorithm visits every level before it finds the optimum. For $i\in [T]$, let $t_i$ be the first points in time where the level exceeds $i-1$, i.e., $\ell(t_{i}-1)=i-1$ and $\ell(t_{i})\geq i$. 

First note that the set $A_i$ only impacts the fitness if either $\ell(t) \geq i-1$, or indirectly if $d(B_i,y^{(t)}) \leq \eps$. As long as neither of them happens, $A_i$ is just a random subset drawn from $[n]$ that is independent of the mutation/crossover and selection decisions of the algorithms. Therefore, it is exponentially likely that $d(A_i,t) = d([n],t) \pm \xi$. For the same reason, it is exponentially likely that $d(B_i,t) = d(A_i,t) \pm \xi$, until the time that we hit level $i$ for the first time. In particular, since we will show that whp $d([n],t) \geq \eps + 3\xi$ for an exponential number of rounds, for each $t < t_i$ it is exponentially likely that $d(A_i,t) \geq \eps+2\xi$ for all $t < t_{i-1}$. Moreover, we claim that it is exponentially likely that $d(B_i,y^{(t)}) > \eps$ for any such round. Indeed, either $s^{(t)} < \beta \xi n$ in which case $d(B_i,y^{(t)})$ decreases by less than $\xi$. Or $s^{(t)} \geq \beta \xi n$. In this case, if $d(B_i,t) \leq 1/3$ then in expectation $d(B_i,y^{(t)}) > d(B_i,t)$ since there are more one-bits than zero-bits in $B_i$, and the exponential tail bound follows from the Chernoff bound. If $d(B_i,t) > 1/3$ then a similar argument shows that it is exponentially likely that $d(B_i,y^{(t)}) \geq 1/4$. Note that if $\rho$ is small enough, we can afford to take a union bound over all values of $t$ and all indices $i$ (which gives $T^2 = e^{2\rho n}$ combinations), and infer that whp $\ell(t)$ only increases in steps of size one, and only in rounds in which $d(A_i,t) \leq \eps +\xi$. In particular, for the rest of the proof, at level $\ell=i-1$ we may ignore all indices $\geq i+1$, as long as we show inductively that $d([n], t) \geq \eps + 3 \xi$ throughout the process. 

For later use, we note that the same argument as above shows that it is exponentially unlikely that an offspring with more than $\beta\xi n$ bit flips is accepted, if $d(A_{\ell+1},t) \leq 1/3$: either $s^{(t)} < \beta \xi n$, in which case the statement is true. Or $s^{(t)} \geq \beta \xi n$ in which case it is exponentially likely that $d(A_{\ell+1},y^{(t)}) > d(A_{\ell+1}, t)$. In particular, the densities of all $A_i$ and $B_i$ changes in each step by at most $\xi$. Note that this also shows that at time $t_i$ we have $d(A_{i+1},t_i) = d([n],t_{i-1}) \pm \xi = \alpha d(A_i,t_{i-1}) + (1-\alpha)d(A_i,t_{i-1}) \pm \xi$, if the density in $[n]$ is at most $1/3$. By a similar argument, the density cannot drop from $d(A_i,t) \geq 1/3$ to $d(A_i,t+1) \leq 1/4$ in one step, and similarly for $d(B_i)$, $d(R_i)$, and $d([n])$. 

By the same argument, if $\eps \leq d(A,t) \leq 1/3$ then the probability that $x^{(t+1)}$ differs from $x^{(t)}$ by at least $s$ bits drops exponentially in $s$. This means that the random variables $d(A,t)$ and $d(R,t)$ satisfy the step size condition from the negative drift theorem, Theorem~\ref{thm:tailbounds}. Hence, it suffices to bound the drift. For example, if the drift of $d(A)$ towards zero is at most $\kappa/n$ for some constant $\kappa$, then it is exponentially unlikely that $d(A)$ decreases from $\eps + (1-\alpha)\delta$ to $\eps$ in less than $((1-\alpha)\delta/\kappa + \xi) \cdot n$ steps, for any constant $\xi>0$. Moreover, we know that $t_{i+1}$ will occur at some point with $\eps - \gamma \leq d(A_i) \leq \eps + \gamma$, so it suffices to study $d(A_i)$ in the regime between $\eps -\gamma$ and $\eps+\delta$. (It cannot happen that the density becomes larger, since $d(A_i)$ can never increase if the level stays the same.) Finally, we remark that it was shown in~\cite[Lemma 9]{lengler2016drift} that when $d(A_1,t) < \eps +\delta$ for the first time, then whp $d([n],t) \geq \eps + \delta -2\xi$, which gives us the starting condition for the first round. The lemma was formulated for the \ooea, but it builds on an argument by J\"agersk\"upper~\cite{jagerskupper2011combining} that holds for any unbiased algorithm. This justifies all arguments in the sketch at the beginning of the proof, except that we need to compute and compare the drift of $d(A_i,t)$ and $d(R,i,t)$. \smallskip

We compute tight bounds on the drift, so that we can reuse them later in the proof of (b). First we compute the drift of $d(A_i,t)$ for level $\ell(t) = i-1$ (i.e., $A_i$ is the current ``hot topic''). Note that the density in $B_{i-1}$ is not allowed to drop below $\eps$, since otherwise the level would decrease. However, recall that we assumed $\beta$ to be small compared to $\eps$ and $\delta$, so the contribution from this effect will turn out to be negligible. We consider the regime $\eps -\xi \leq d(A_i) \leq \eps+\delta$. 
Assume that $s^{(A_i)} = \sigma$, i.e., exactly $\sigma\in \N$ bits are flipped in $A_i$. Recall that we may assume $\sigma \leq \xi\alpha n$, since otherwise the offspring is exponentially likely to be rejected. The positions of these bits are uniformly at random in $A_i$, since the operator is unbiased. Note that a single bit has probability $\eps \pm O(\delta)$ to be a zero bit. The density $d(A_i)$ can only decrease if $s_{01}^{(A_i)} > s_{10}^{(A_i)}$, which implies $s_{01}^{(A_i)} \geq (\sigma+1)/2$. If we draw the positions successively, then each bit still has probability $p=\eps \pm O(\delta) \pm O(\xi) = \eps \pm O(\delta)$ to be a zero bit. Therefore, 
\begin{align*}
\Pr[s_{01}^{(A_i)} > s_{10}^{(A_i)} \mid s^{(A_i)} = \sigma] \leq \Pr[\Bin(\sigma,p) > \sigma/2] \leq \eps^{\Omega(\sigma)} 
\end{align*}
by the Chernoff bound. For $\sigma = O(1)$, we may bound more precisely $\Pr[s_{01}^{(A_i)} > s_{10}^{(A_i)} \mid s^{(A_i)} = \sigma] \leq O(\eps^{\lceil(\sigma+1)/2\rceil})$. In particular, the cases $\sigma \geq 2$ contribute in expectation at most $O(\eps^2)$ bits. Note that one bit changes the density of $A_i$ by $1/(\alpha n)$. Therefore, writing $p_\sigma := \Pr[s= \sigma]$, the drift is
\begin{align}\label{eq:hardstrong3}
\Delta := \Delta(\eps) & := \E[d(A_i,t)-d(A_i,t+1) \mid d(A_i,t) \in [\eps-\zeta, \eps+\delta]] \nonumber\\
& = \frac{\eps \Pr[s^{(A_i)}=1] + O(\eps^2+\delta) }{\alpha n}\nonumber \\
& \stackrel{(*)}{=} \frac{\eps\sum_{\sigma=1}^{\infty}  p_\sigma \cdot \sigma \alpha (1-\alpha)^{\sigma-1} \pm O(\eps^2+\delta)}{\alpha n} \nonumber \\
& = \frac{\eps \E[s(1-\alpha)^{s-1}] \pm O(\eps^2+\delta)}{n},
\end{align}
where (*) is justified since the sum converges, and since for any constant $\sigma$ the term $\sigma\alpha(1-\alpha)^{\sigma-1}$ approximates the true probabilities $\Pr[s^{(A_i)}=1 \mid s= \sigma]$ up to a $(1\pm o(1))$ factor. More precisely, there is a constant $\sigma'$ such that $\sum_{\sigma=1}^{\infty}  p_\sigma \sigma \alpha (1-\alpha)^{\sigma-1} = (\sum_{\sigma=1}^{\sigma'}  p_\sigma \sigma \alpha (1-\alpha)^{\sigma-1}) \pm \eps$ and $\Pr[s^{(A_i)}=1] = (\sum_{\sigma=1}^{\sigma'}  p_\sigma \Pr[s^{(A_i)}=1 \mid s= \sigma]) \pm \eps$. For all $\sigma \in [1,\sigma']$ we have $\Pr[s^{(A_i)}=1 \mid s= \sigma] = (1\pm \eps)\sigma\alpha(1-\alpha)^{\sigma-1}$ if $n$ is sufficiently large. Thus is $n$ is sufficiently large we have $\Pr[s^{(A_i)}=1] = \sum_{\sigma=1}^{\infty}  p_\sigma \sigma \alpha (1-\alpha)^{\sigma-1} \pm 3\eps$

Next we turn to the drift of $d(R_i,t)$ for level $\ell(t) = i-1$. We consider the regime $\eps-\xi \leq d(A_i,t), d(R_i,t) \leq \eps + \delta$. We first study the cases in which the density in $A_i$ does not decrease. Note that then the offspring is not accepted if it strictly increases $d(R_i)$. Therefore, this case can only contribute positively to $R_i$.  If no bit in $A_i$ is flipped, then the case is similar as for $d(A_i)$, and single-bit flips dominate the drift, with the other cases contributing $O(\eps^2)$ bits in expectation. Similarly, the case that at least one zero bit in $A_i$ and at least one zero bit in $R_i$ are flipped contribute in expectation $O(\eps^2)$ bits. Note that these are the only cases in which the density in $A_i$ does not decrease \emph{and} $d(R_i)$ increases. So let us consider the case that no bit in $A_i$ is flipped, and a single bit in $R_i$ is flipped. In this case the total number of bit flips is one. Also, the probability that a single bit flip hits a zero-bit in $R_i$ is $(1-\alpha)\eps + O(\delta)$. Therefore, all mentioned cases contribute a drift of 
\begin{align*}
\Delta_1' & := \frac{\eps(1-\alpha) \Pr[s=1] + O(\eps^2) + O(\delta)}{(1-\alpha)n} \\
& = \frac{\eps \Pr[s=1] + O(\eps^2) + O(\delta)}{n}
\end{align*}
towards zero. 

Now we consider the case that the density in $A_i$ decreases. If the density $d(B_{i-1})$ stays below $\eps$, then the search point is accepted regardless of what happens in $R_i$. In this case $d(R_i)$ will increase in expectation since the operation is unbiased and $d(R_i) \leq 1/2$. As before, due to their low probability all cases with $s_{01}^{(A_i)}\geq 2$ contribute only an expected $O(\eps^2)$ number of bits. So we need to study the case $s^{(A_i)} = s_{01}^{(A_i)}=1$. For a lower bound, we estimate the drift in the case that additionally $s^{(A_i \cup B_i)}=1$, as this implies that the offspring is accepted. For an upper bound, we pessimistically assume that all offspring with $s^{(A_i)} = s_{01}^{(A_i)}=1$ are accepted.

Note that the number of zero-bits in $R_i$ increases by $s_{10}^{(R_i)}- s_{01}^{(R_i)}$. Moreover, $\E[s_{10}^{(R_i)}- s_{01}^{(R_i)}] = (1-O(\eps))\E[s^{(R_i)}]$, and the same remains true if we condition on events in $A_i$ and $B_i$ on both sides. Each bit changes the density $d(R_i)$ by $1/((1-\alpha)n)$. Therefore, these cases contribute a drift away from zero of at least 
\begin{align}
\Delta_2' & := \frac{(1-O(\eps))}{(1-\alpha)n} \cdot \Pr[s_{10}^{(A_i \cup B_{i-1})}=0,s_{01}^{(A_i)} = 1] \label{eq:hardstrong7a}\\
& \qquad \qquad \qquad \cdot \E[s^{(R_i)} \mid  s_{10}^{(A_i \cup B_{i-1})} =0,s_{01}^{(A_i)} = 1] \label{eq:hardstrong7b}\\
& = \frac{(1-O(\eps))}{(1-\alpha)n} \cdot \sum_{\sigma = 1}^{\infty}p_\sigma (\sigma-1) \cdot (\sigma \eps \alpha) \cdot (1-\alpha\pm \beta)^{\sigma-1}\nonumber\\
& \stackrel{(*)}{=} \frac{(1-O(\eps))\eps\alpha}{(1-\alpha)n} \cdot \sum_{\sigma = 1}^{\infty}p_\sigma \sigma(\sigma-1) (1-\alpha)^{\sigma-1}\nonumber\\
& = (1-O(\eps))\frac{\eps\alpha \E[s(s-1)(1-\alpha)^{\sigma-1}]}{(1-\alpha)n}\nonumber,
\end{align}
where (*) follows since the sum converges as $(1-\alpha \pm \beta) <1$. Therefore, if we make $\beta$ sufficiently small, the sum changes by at most $\eps$ if we replace $(1-\alpha \pm \beta)$ by $(1-\alpha)$. Note that we obtain an upper bound on the drift by replacing $s_{10}^{(A_i \cup B_{i-1})} =0$ by $s_{10}^{(A_i)} =0$ in~\eqref{eq:hardstrong7a} and~\eqref{eq:hardstrong7b}, which still leads to the same next line. Thus our estimate of $\Delta_2'$ is both a lower and an upper bound. In particular, we have a tight estimate for the following fraction, which will turn out to be useful.
\begin{align*}
\Phi' & := \frac{(1-\alpha)}{\alpha} \cdot\frac{\Delta'}{\Delta} = \frac{1-\alpha}{\alpha}\cdot \frac{\Delta_2'- \Delta_1'}{\Delta} \\
& = \frac{ \E[s(s-1) (1-\alpha)^{s-1}]- \tfrac{(1-\alpha)}{\alpha}\Pr[s=1] \pm O(\eps+\tfrac{\delta}{\eps})}{\E[s(1-\alpha)^{s-1}] \pm O(\eps+\tfrac{\delta}{\eps})}
\end{align*}
So up to the error terms, this term equals $\Phi$ as defined in~\eqref{eq:hard0}. In particular, $\Phi' = \Phi+O(\eps+\delta/\eps)$ unless the two terms in the numerator cancel out approximately. Note that this only happens if $\Phi$ is very small, so it is only relevant for part (b). In this case we still have $\phi' \leq 1/2$, which is sufficient for our argument. However, for the sake of readability we will suppress this case and assume that $\Phi' = \Phi+O(\eps+\delta/\eps)$.\smallskip


To conclude the proof, we proceed as in the sketch. For (a), assume that we enter level $i$ with density $d(A_i), d(R_i)\geq   \eps + (1-\alpha)\delta - O(\xi)$. Then the number of steps to decrease $d(A_i)$ to $ \eps +\xi$ is at least $\tau := (1-\alpha)\delta/\Delta - \xi$ with exponentially small error probability. On the other hand, it is exponentially likely that $d(R_i)$ increases in $\tau$ steps by at least $\tau \cdot \Delta'-\xi$ (or until it hits $\eps+\delta$), where $\Delta' := \Delta_2'-\Delta_1'$ is the drift of $R_i$. Moreover, once $d(R_i)$ has reached $\eps+\delta$, it will need an exponential number of steps to fall back below $\eps+\delta-\xi$, unless the level increases. It remains to check that $\tau \Delta' > \alpha \delta$ with some safety margin. Then we know that when the level increases, we have $d(A_i) = \eps \pm \xi$ and $d(R_i) \geq \eps + \delta - \xi$, which implies that we also enter the next level with density $d(A_{i+1}), d(R_{i+1})= \eps + (1-\alpha)\delta \pm O(\xi)$. To check $\tau \Delta' > \alpha \delta$, we use
\begin{align*}
\frac{\tau \Delta'}{\alpha\delta}  = \frac{1-\alpha}{\alpha}\cdot \frac{\Delta'}{\Delta}+O(\tfrac{\xi}{\delta}) = \Phi \pm O(\eps+\tfrac{\delta}{\eps}+\tfrac{\xi}{\delta}) > 1+\zeta',
\end{align*}
if $\eps$, $\delta/\eps$ and $\xi/\delta$ are sufficiently small. This concludes the proof of (a).\medskip

\noindent\emph{(b).} For (a) we have already proven upper and lower bounds on the drift of $A_{i+1}$ and $R_{i}$ in the regime $\eps-\xi \leq d(A_i)\leq \eps + \delta$ and $\eps \leq d(R_i) \leq \eps+\delta$. We will extend this range in two ways. First we observe that, since $\eps>0$ was arbitrary, the same result also extends to the case where we replace $\eps$ by any other $\eps'$, as long as $\eps'$ and $\delta/\eps'$ are still sufficiently small. In particular, it will hold for any $\eps' \in [\eps-\xi, \tilde \eps]$, if $\tilde \eps$ is sufficiently small. 
Note that the fraction $\alpha/(1-\alpha) \cdot \Delta/\Delta'$ is constant for any $\eps'$, up to the error terms.

Secondly we extend the range to the case that $R_{i}$ is larger than $\eps+\delta$. The drift of $d(A_i)$ is unaffected by the bits in $R_{i}$, so we only need to study the drift in $d(R_{i})$. We will show that the drift away from zero $\E[d(R_i,t+1)-d(R_i,t)]$ is decreasing in $d(R_{i}\setminus B_{i-1})$. For the bits in $R_i \cap B_{i-1}$, since there at most $\beta n$ of them they can contribute at most an $O(\beta)$ error term, which is much smaller than the $\Theta(\eps)$ main terms that we have for the drift, so we can swallow the bits in $R_i \cap B_{i-1}$ in the error terms. 

So let $j\in R_{i}\setminus B_{i-1}$, and consider any two $x, x' \in \{0,1\}$ that coincide in all bits except for $j$, such that $x_j = 0$ and $x_j'=1$. Moreover, we will assume that $x,x'$ are on the same level $i$. Now we draw the number of bit flips $s$, and a set $S$ to be flipped with size $|S|=s$. We compare the change of $d(R_{i+1})$ if we flip $S$ in either $x$ or $x'$. 
Let us call the offspring $y$ and $y'$, respectively. If $j\not\in S$ then the two case are identical, and $d(R_{i})$ will change by exactly the same amount (possibly by zero) for $x$ and $x'$. So let us assume $j\in S$. If the level decreases, or the density of $A_i$ increases, then the offspring will be rejected, and $d(R_{i+1})$ does not change in either case. If the level increases, or the density of $A_i$ decreases, then the offspring will always be accepted, and $d(R_{i+1})$ will increase by a strictly smaller amount for $x$. Finally, if the level and the density of $A_i$ do not change, then the offspring will be accepted if and only the density in $d(R_{i+1})$ decreases. In other words, every such case contributes non-positively to the drift. However, since $\HT(y)> \HT(y')$ and $\HT(x) < \HT(x')$, we have the implication $\HT(y')-\HT(x') \geq 0 \implies \HT(y)-\HT(x) \geq 0$. Thus if the offspring is accepted for $y'$ then it is also accepted for $y$. Since all contributions are non-positive, this shows that the drift towards zero of $d(R_i)$ is indeed decreasing with $d(R_i\setminus B_{i-1})$. In particular, the estimate $\Delta'$ for the drift (with an additional error term of $O(\beta/\eps)$) is still an upper bound for any $d(R_i) \geq \eps$.\smallskip

Now assume that we could show that at some point $d([n])\leq \tilde\eps/2$. Note that we may assume $d(A_i) \geq \eps-\xi$, since otherwise we would have reached a higher level $\ell \geq i$. Since our bounds on the drift now apply to this case, the drift of the total density $d([n])$ towards zero is at least $\alpha \Delta - (1-\alpha)\Delta' = \Omega(\eps/n)$, since $(1-\alpha)\Delta'/(\alpha\Delta) = \Phi+O(\eps+\delta/\eps + \beta/\eps) \leq 1-\zeta'$ by \eqref{eq:hard2}, and since $\Delta = \Omega(\eps/n)$. Note that other than $d(A_i)$ and $d(R_i)$, the density $d([n])$ does not change if the level $\ell = i-1$ increases. 
In particular, by the negative drift theorem, whp $d([n])$ will stay in the regime $d([n]) \leq \tilde\eps$ for an exponential number of steps, and thus the drift bounds apply until either $d([n]) \leq \eps-\xi$ or until we have reached the maximal level $\ell = T-1$ (since then the condition $d(A_i) \geq \eps-\xi$ may be violated). Again by the negative drift theorem, we will reach $d([n]) \leq \eps-\xi$ in a linear number of steps, unless the maximal level is reached. However, $d([n]) \leq \eps-\xi$ also implies that whp the maximal level is reached. So we have shown that whp the algorithm reaches the highest level in a linear number of steps. 
Afterwards, $d([n])$ can no longer increase, so again by the multiplicative drift theorem whp the algorithm hits the global optimum in time $O(n\log n)$. This concludes the proof modulo the statement that at some point $d([n]) \leq \tilde \eps/2$. \smallskip

To show the latter, we show that even for $i=1$, i.e., on level $\ell=0$, we have $d([n]) \leq \tilde \eps$ before $d(A_1) \leq 2\eps$, if $\eps$ is sufficiently small. We will use a rather crude bound on the drift of $d([n])$ that is always valid as long as $d(A_1), d([n]) \geq 2\eps$. In particular, it also holds for search points far away from the optimum. Note that by the conditions $d(A_1), d([n]) \geq 2\eps$ we may assume that the level does not change. Then, $d([n])$ can only increase if a zero bit in $A_1$ is flipped. If $s$ bits are flipped in total, then the probability that this happens is at most $\Pr[s_{01}^{(A_1)}\geq 1 \mid s] \leq \E[s_{01}^{(A_1)}\mid s] = s\alpha d(A_1)$. Moreover, in this case $d([n])$ can increase at most by $(s-1)/n$. On the other hand, if exactly one bit is flipped, and this is a zero bit, then $d([n])$ decreases by $1/n$. Hence,
\begin{align*}
\E[d([n],t) &- d([n],t+1)] \\
& \geq \frac{\Pr[s= s_{01}^{([n])}=1]}{n} - \sum_{\sigma \geq 1}p_{\sigma}\frac{\sigma\alpha d(A_1)\cdot(\sigma-1)}{n}\\
& = \frac{\Pr[s=1]}{n}\left(d([n])-C\cdot d(A_1)\right),
\end{align*}
where $C := \E[s(s-1)]\cdot\alpha/\Pr[s=1]$ is a constant that does not depend on $\eps, \tilde \eps$. In particular, if $2\eps \leq d(A_1) \leq \tilde \eps/(4C)$ and $d([n]) \geq \tilde \eps/2$ then the drift of $d([n])$ towards zero is at least $\E[d([n],t)- d([n],t+1)]\geq C' \cdot \tilde \eps/n$, where the constant $C' :=  \Pr[s=1]/4 \geq \zeta/4$ does not depend on $\eps$. In particular, if the condition $2\eps \leq d(A_1) \leq \tilde \eps/(4C)$ holds for time at least $n/C'$, then by the negative drift theorem whp $d([n])$ drops below $\tilde \eps/2$ in this time.

So it remains to show that $d(A_1)$ decreases slowly enough. By~\eqref{eq:hardstrong3}, for $2\eps \leq d(A_1) \leq 2\tilde\eps$, the drift of $d(A_1)$ towards zero is at most $C''\cdot d(A_1)/n$ for a suitable constant $C'' >0$ that does not depend on $\eps, \tilde\eps$. In particular, in any range $d(A_1) \in [\eps',2\eps']$ the drift is at most $2C''\eps'/n$, and by the negative drift theorem it is exponentially unlikely that $d(A_1)$ decreases from $2\eps'$ to $\eps'$ in time less than $n/(4C'')$. Since this holds for any $\eps'$, whp the time in which $d(A_1)$ decreases from $\tilde \eps/(4C)$ to $2\eps$ is at least $\log_2(\lfloor \tilde \eps/(4C)/(2\eps)\rfloor)\cdot n/(4C'')$, which is larger than $n/C'$ if $\eps$ is small enough. This shows that whp the regime $2\eps \leq d(A_1) \leq 2\tilde\eps$ prevails long enough such that $d([n])$ drops below $\tilde \eps /2$. As this was the last missing ingredient, this proves (b).\medskip

\noindent\emph{(c).} The statement of (c) is much more trivial than the others. The first claim for simply follows because the algorithm is indifferent against steps which are not improving. The only difference is that we can no longer infer the probability to make an improving step, but that is irrelevant for (a), and covered by the additional condition for (b). For the second statement, by the same argument as for (b), whp the algorithm reaches $d([n]) \leq \zeta/2$ after linear time, and stays in the range $d([n]) \leq \zeta$ for long enough afterwards. So we may assume that we are in this range. Whenever we accept an offspring, then we have $s_{01} \geq 1$ and $s_{10} \leq s-1$. Therefore, the drift of $d([n])$ towards zero is at least
\begin{align*}
\E&[d([n,t])-d([n,t+1])] \\
& \geq \Pr[\HT(y) > \HT(x)]\cdot \frac{1}{n}(2-\E[s\mid \HT(y) > \HT(x)]) \\
& \geq \Pr[s=1]d([n],t) \cdot \frac{\zeta}{n} = \Omega\left(\frac{d([n],t)}{n}\right).
\end{align*}
The rest follows as in (b) from the multiplicative drift theorem. \medskip

It remains to prove the statement on the adversary. However, since the probability that $f(x') > f(x)$ is $O(\eps)$, the actions of the adversary only add a term $O(\eps \eta/n)$ to the drift of $d(A_i)$ and $d(R_i)$. If $\eta>0$ is sufficiently small (depending on $\alpha$, but independent of $\tilde \eps, \eps, \delta, \beta,\ldots$), this error term is negligible. Moreover, for (a) the tail bound on $\tau$ allows us to still apply the negative drift theorems. For (b), we use the negative drift theorem twice, but we don't need exponential tail bounds. The first time we use it to show that once the algorithm is at a search point of density $d([n]) \leq \tilde \eps/2$, whp it does not climb back to $d([n]) \geq \tilde \eps/2$ in the next $O(n\log n)$ steps. However, in this time the adversary may whp never alter more than $n^{1-\eta}$ bits. Therefore, the statement still follows by applying the negative drift theorem with step size bounded by $n^{1-\eta}$, i.e., we apply Theorem~\ref{thm:tailbounds} to $X_t := n \cdot d([n],t)$ and $r(n) = n^{-\eta}$. In the second application of the negative drift theorem, we estimate the time in which $d(A_1)$ decreases from $2\eps$ to $\eps$, and the time in which $d([n])$ drop below $\tilde \eps/2$. As before, both estimates still hold with high probability in the presence of the adversary, since we may assume that there are no steps of size $n^{1-\eta}$. For (c) we apply the negative drift theorem in the same way as for (b). This concludes the proof.
\end{proof}

\section{Concrete Results for HotTopic}
\label{sec:concrete}

It turns out that Theorem~\ref{thm:generic} suffices to classify the behaviour on \hottopic for all algorithms that we study. On the first glance, this may seem surprising, since some of them are population-based, while Theorem~\ref{thm:generic} explicitly requires population size one. Nevertheless, we will see that it implies the following theorem.

\begin{theorem}[\hottopic, Concrete Results]\label{thm:concrete}
Let $\delta >0$. We assume that $\mu, \lambda, c = \Theta(1)$ and $\Pr[\mathcal D = 1] = \Omega(1)$, except for the \ollga, for which we replace the condition on $c$ by $c\gamma = \Theta(1)$. Let $c_0 = 2.13692..$ be the smallest constant for which the function $c_0 x-e^{-c_0(1-x)}-x/(1-x)$ has a solution $\alpha \in [0,1]$. For all $\alpha \in (0,1)$, with high probability each of the following algorithms optimises the function $\hottopic_{\alpha,\beta,\rho,\eps}$ with parameters $\beta,\rho,\eps$ as in~\eqref{eq:constantsmonotone} in time $O(n\log n)$.
\begin{itemize}
\item The \olea with $c \leq c_0-\delta$.
\item The \moea with $c \leq c_0-\delta$.
\item The \moga with arbitrary $c = \Theta(1)$ if $\mu = \mu(c)$ is sufficiently large.
\item The \ollga with $c\gamma \leq c_0-\delta$.
\item The \folea with $m_2/m_1 \leq 1-\delta$; more generally, the \folea with any distribution that satisfies~\eqref{eq:hard2} for $s\sim\mathcal D$, as well as $\Pr[\mathcal D=1] = \Omega(1)$.\footnote{\label{footnote:fastconcrete}Note that this is not a trivial consequence of Theorem~\ref{thm:generic}, since~~\eqref{eq:hard1},~\eqref{eq:hard2} are conditions on the distribution for the best of $\lambda$ offspring, while the condition here is on the distribution $\mathcal D$ for generating a single offspring.} 
\item The \fmoea with parameters as in the preceding case, if additionally $\Pr[\mathcal D = 0] = \Omega(1)$.
\item The \fmoga with arbitrary $\mathcal D$ with $\Pr[\mathcal D=0]= \Omega(1)$, if $\mu = \mu(\mathcal D)$ is sufficiently large.
\end{itemize}
On the other hand, for $\alpha_0 = 0.237134..$, with high probability each of the following algorithms needs exponential time to optimise the function $\hottopic_{\alpha_0,\beta,\rho,\eps}$ with parameters $\beta,\rho,\eps$ as in~\eqref{eq:constantsmonotone}.
\begin{itemize}
\item The \olea with $c \geq c_0+\delta$.
\item The \moea with $c \geq c_0+\delta$.
\item The \moga with $c \geq c_0+\delta$ if $\mu = \mu(c)$ is sufficiently small.\footnote{\label{footnote:mu1}This statement follows trivially from the other results by setting $\mu=1$, and it is listed only for completeness.}
\item The \ollga with $c\gamma \geq c_0+\delta$.
\item The \folea with any distribution satisfying~\eqref{eq:hard1} for $s\sim\mathcal D$.$^{\ref{footnote:fastconcrete}}$ In particular, this includes the following cases.
\begin{itemize}
\item The \folea with $m_2/m_1 \geq 1+\delta$, if the probability to flip a single bit is sufficiently small compared to $s_0:= \min\{\sigma \in \N \mid m_{2,\leq \sigma} \geq (1+\delta/2) m_1\}$, where $m_{2,\leq \sigma} := \sum_{i=1}^{\sigma}\Pr[\mathcal D = i]i(i-1)$ is the truncated second falling moment.
\item The \folea with any power law distribution with exponent $\kappa \in (1,2)$, i.e. $\Pr[\mathcal D \geq \sigma] =\Omega(\sigma^{-\kappa})$. 
\item The \folea with $\Pr[\mathcal D =1] \leq \tfrac49 \cdot \Pr[\mathcal D \geq 3]-\delta$.
\end{itemize}
\item The \fmoea in all preceding cases for \folea, if additionally $\Pr[\mathcal D = 0] = \Omega(1)$.
\item The \fmoga in all preceding cases for \folea, if the population size $\mu = \mu(\mathcal D)$ is sufficiently small.$^{\ref{footnote:mu1}}$
\end{itemize}
\end{theorem}

\begin{remark}\label{rem:functionanalysis}
To see why the inequality $f(c,x) := c x-e^{-c(1-x)}-x/(1-x) \geq 0$ has a solution $x \in [0,1]$ if and only if $c \geq c_0$, it suffices to observe that the derivative with respect to $c$ is $\partial f/\partial c (c,x)= x+(1-x)e^{-c(1-x)}>0$. Hence, $f(c,x)$ is strictly increasing in $c$. The value of $c_0$, and the unique $\alpha_0$ with $f(c_0,\alpha_0)=0$ can numerically be computed by observing that we must have $f(c_0,\alpha_0) = \partial f/\partial x(c_0,\alpha_0) = 0$. In particular, this implies $0 = c_0 f(c_0,\alpha_0) - \partial f/\partial x(c_0,\alpha_0) = (1-c_0(1-\alpha_0)+c_0^2(1-\alpha_0^2)\alpha_0)/(1-\alpha_0)^2 =: \tilde f(c_0,\alpha_0)$. This is a quadratic equation in $c_0$ and has the two solutions $c_0 = g_{\pm}(\alpha_0) := (1\pm \sqrt{1-4\alpha_0})/(2\alpha_0(1-\alpha_0))$ for $\alpha_0 \in (0,1/4]$, and no solution otherwise. We can plug this term into the definition of $f(c,x)$, and obtain that $\alpha_0$ is a root of $h_{\pm}(x):=f(g_{\pm}(x),x)$. The function $h_-$ is strictly increasing in the interval $[0,1/4]$ (the derivative can be checked to be positive in $(0,1/4]$) from $h_-(0)= -1/e< 0$ to $h_-(1/4)= 1/3-1/e^2 >0$, and thus it has a single zero in $[0,1/4]$, which is $\alpha_0 = 0.237134..$. The function $h_+$ is strictly decreasing from $h_+(0)= 1$ to $h_+(1/4)= 1/3-1/e^2 >0$, and thus has no zero. Finally $c_0$ can then be computed as the root of $\tilde f(c_0,\alpha_0)=0$. 
\end{remark}

\begin{remark}\label{rem:Zipf}
For the fEA's we remark that the interesting regime $\kappa \in [2,3)$ is not excluded by the negative results in Theorem~\ref{thm:concrete}, if $\Pr[\mathcal D]$ is sufficiently large. In particular, a calculation with Mathematica\texttrademark\ shows that the Zipf distribution\footnote{i.e., $\Pr[\mathcal D = k] = k^{-\kappa}/\zeta(\kappa)$, where $\zeta$ is the Riemann $\zeta$ function.} with exponent $\kappa \geq 2$ satisfies~\eqref{eq:hard2} for all $\alpha \in (0,1)$. However, note that this holds only if the distribution is \emph{exactly} the Zipf distribution; changing any probability even by a constant factor may lead to exponential runtimes. Moreover, it is rather questionable whether the Zipf distribution is efficient for \emph{all} monotone functions, as $m_2/m_1 = \infty$ in this regime.
\end{remark}

\begin{proof}[Proof of Theorem~\ref{thm:concrete}]
All results will be applications of Theorem~\ref{thm:generic}. We first outline the general strategy, for concreteness in the case of the \olea and \folea. To apply Theorem~\ref{thm:generic} directly, we would need to analyse the distribution of the number of bit flips in the best offspring in each generation. Note crucially that this may be very different from the distribution $\mathcal D$ that creates a single offspring. However, the key feature of Theorem~\ref{thm:generic} is that it allows us to restrict our analysis to the case when the parent has at most $\zeta n$ zero bits. Still the number of bit flips in the fittest offspring is not the same as $\mathcal D$, but we can use a neat trick.  We ``modify'' the algorithm by choosing the winner offspring in a slightly different way. If none of the offsprings flips a zero-bit, then we do not compare the \emph{fittest} offspring with the parent $x$, but rather a \emph{random} offspring. Otherwise, we proceed as usual with the fittest offspring. Note that this little thought experiment does not change the behaviour of the algorithm, since in the former case \emph{all} offspring are either identical to $x$, or have strictly worse fitness than $x$. So the algorithm just stays with the parent. However, if we call our weirdly selected winner offspring $y$, then suddenly the distribution $\mathcal D'$ of $y$ is very similar to the distribution $\mathcal D$ of a random offspring, since most of the time we do not flip any one-bits. We will be able to use the same trick for all the algorithms above, even for the population-based ones.

This construction would do the trick, except that it is not unbiased. However, note that if there is \emph{exactly one} offspring $y^{(k)}$ which is fitter than $x$ then the distribution of $y^{(j)}$ is unbiased conditioned on $f(y^{(j)}) > f(x^{(j)})$, i.e., the distribution of $y^{(j)}$ is the same as the (unbiased) distribution $\mathcal D$ of a single offspring, conditioned on this offspring being unbiased. Therefore, we do have an unbiased distribution except for the case that there are at least two offspring which are fitter than $x$. We will attribute all these to the adversary. Thus, we need to show that the probability $\Pr[\text{at least two fitter offspring} \mid \text{at least one fitter offspring}] <\eta = \eta(\zeta',\alpha)$, and that we have a tail bound on the number of bit flips in this case. The tail bound follows as in the proof of Theorem~\ref{thm:generic} by observing that the probability to flip at least as many zero-bits as one-bits in $A_i$ decreases exponentially in the number $s$ of bit flips, if $d(A_i) \leq 1/3$. In particular, the expected number of bit flips in a fitter offspring is $O(1)$. Hence the probability to generate a better offspring is at most $O(\zeta)$, and so is the probability to generate a \emph{second} fitter offspring. We can make this probability smaller than $\eta$ by choosing $\zeta$ sufficiently small, since $\eta = \eta(\zeta',\alpha)$ is independent of $\zeta$. This shows that the adversary is sufficiently limited.\smallskip

Before we proceed to the individual algorithms, we first show in general why it suffices if $\mathcal D' \to \mathcal D$ weakly. Let us denote by $s$ and $s'$ a random variable from $\mathcal D$ and $\mathcal D'$, respectively, and let us denote $p_\sigma := \Pr[s= \sigma]$ and $p_\sigma' := \Pr[s'= \sigma]$. Assume that for each $\sigma \in \N$ there is $\zeta_0 = \zeta_0(\sigma) >0$ such that for all $0<\zeta \leq \zeta_0$ we have $p_\sigma = p_\sigma' \pm \xi$. Then we need to show that the value of $\Phi$ is approximately the same for $s$ and $s'$. For convenience, we repeat the definition of $\Phi$:
\begin{align*}
\Phi = \Phi(x) = \frac{\E[s(s-1) (1-\alpha)^{s-1}]}{\E[s(1-\alpha)^{s-1}]} - \frac{\tfrac{(1-\alpha)}{\alpha}\Pr[s=1]}{\E[s(1-\alpha)^{s-1}]}.
\end{align*}
We define $\Phi'$ analogously with $s'$ instead of $s$. Consider the expectations in $\Phi$. We can approximate each of them up to an error of $\xi$ if we consider the contribution of the case $\sigma \leq \sigma_0$ for the expectations. More precisely, for each $\xi >0$ there is some constant $\sigma_0 \in \N$ such that $\E[s(1-\alpha)^{s-1}] = \sum_{\sigma = 0}^{\sigma_0}p_\sigma \sigma(1-\alpha)^{\sigma-1} \pm \xi$, and similarly for $\E[s(s-1)(1-\alpha)^{s-1}]$. Now we use our assumption that for each $\sigma \in \{0,\ldots,\sigma_0\}$ there is $\zeta_0 = \zeta_0(\sigma)>0$ such that $p_\sigma = p_\sigma' \pm \xi$ whenever $0<\zeta< \zeta_0$. Since we only want to achieve this for a constant number $\sigma_0$ of values, we can choose $\zeta_0 := \min\{\zeta_0(\sigma) \mid \sigma \in \{0,\ldots,\sigma_0\}\}>0$, and we obtain that $p_\sigma = p_\sigma' \pm \xi$ holds for all $\sigma \in \{0,\ldots,\sigma_0\}$ simultaneously. Therefore,
\begin{align*}
\E[s(1&-\alpha)^{s-1}]  = \sum_{\sigma = 0}^{\sigma_0}p_\sigma \sigma(1-\alpha)^{\sigma-1} \pm \xi \\
& = \sum_{\sigma = 0}^{\sigma_0}(p_\sigma' \pm \xi) \sigma(1-\alpha)^{\sigma-1} \pm \xi \\
& = \sum_{\sigma = 0}^{\sigma_0}p_\sigma' \sigma(1-\alpha)^{\sigma-1}\pm \xi\cdot (1+\sum_{\sigma = 0}^{\sigma_0} \sigma(1-\alpha)^{\sigma-1})\\
& = \E[s'(1-\alpha)^{s'-1}] \pm \xi \cdot\left(2+ \sum_{\sigma = 0}^{\infty}\sigma(1-\alpha)^{\sigma-1}\right).
\end{align*}
Since the latter sum converges, we find that we can make the error term arbitrarily small by making $\xi>0$ sufficiently small. The same applies to $\E[s(s-1)(1-\alpha)^{s-1}]$. Since we can approximate each of these terms with arbitrary precision, and since all terms are finite and positive, we can make the error $\Phi-\Phi'$ arbitrarily small by choosing $\zeta>0$ small enough. In particular, if $\Phi(x) < 1-\delta$ then $\Phi' < 1-\delta/2$ for $\zeta$ small enough, and we can apply Theorem~\ref{thm:generic}.\smallskip

Now we turn more concretely to \olea, and \folea. In fact, the \olea ist just a special case of the \folea, where the number of bit flips is given by the binomial distribution $\Bin(n,c/n)$, which converges to $\Poi(c)$ for $n\to \infty$. We first assume that $m_1 < \infty$. Recall that we consider the case that the parent $x$ has at most $\zeta n$ zero bits. Then the probability that at least one of the $\lambda$ offspring hits at least one zero-bit is at most $\Pr[\text{hit zero-bit}] \leq \E[\text{zero-bit flips}] = \zeta \lambda m_1$. Hence, for every $\sigma \in \N$ we have $p_\sigma' = p_\sigma \pm \zeta\lambda m_1$. As outlined above, this implies that $\Phi'$ comes arbitrarily close to $\Phi$ if $\zeta$ is small enough. For the other case, $m_1 = \infty$, fix $\xi>0$, and choose $\sigma_0\in \N$ so large that $\Pr[s > \sigma_0] \leq \xi/(2\lambda)$. Then by a union bound, the probability that at least one offspring flips more than $\sigma_0$ bits is at most $\xi/2$. On the other hand, if $s \leq \sigma_0$ for all offspring then as in the previous case the probability to hit at least one zero-bit is at most $\zeta\lambda\sigma_0 \leq \xi/2$, where the latter inequality is true for all $\zeta \leq 2\lambda \sigma_0\xi$. With this choice, for every $\sigma \in \N$ we have $p_\sigma = p_\sigma' \pm \xi$, as required. Thus we may evaluate $\Phi$ with respect to $\mathcal D$ instead of $\mathcal D'$. 

Before we evaluate $\Phi$, we remark that for the \ollga, we can use almost the same argument with a slightly different construction of the winner offspring. For each of the $\lambda$ offspring, we do $\lambda$ crossover with the parent. If none of the $\lambda^2$ crossover offspring has a flipped zero-bit compared to $x$, then we chose a random crossover offspring. Otherwise we choose the offspring as in the algorithm, i.e., we first pick the fittest mutation offspring $z$, and then pick the fittest crossover offspring of $z$. Since each of the $\lambda^2$ crossover offspring has in expectation $c\gamma$ flipped bits, the probability that there is some crossover offspring with a flipped zero bit is at most $\zeta \lambda^2 c\gamma$. Since this becomes arbitrarily small as $\zeta$ becomes small, the same argument applies, and we may evaluate $\Phi$ with respect to $\mathcal D$ instead of $\mathcal D'$.  

It thus remains to evaluate $\Phi$ for various $\mathcal D$, and check that $\Phi \geq 1-\zeta$ or $\phi \leq 1+\zeta$. For the \olea and the \ollga we have a Poisson distribution $\Poi(c)$. We use Mathematica\texttrademark\ to evaluate $\E[s(1-\alpha)^{s-1}]= ce^{-\alpha c}$, $\E[s(s-1)(1-\alpha)^{s-1}]= (1-\alpha)c^2 e^{-\alpha c}$, and $\Pr[s=1] = ce^{-c}$, which leads to 
\begin{align*}
\Phi= \Phi(\alpha,c) = \frac{1-\alpha}{\alpha}(c\alpha-e^{-(1-\alpha )c}).
\end{align*}
We want to study whether there is $\alpha \in [0,1]$ such that $\Phi(\alpha,c) \geq 1$. This is the case if and only if there is an $\alpha$ such that the function $f(\alpha,c) := c\alpha - e^{-(1-x)\alpha}-\tfrac{\alpha}{1-\alpha}$ takes non-negative values. For constant $c$, the function is negative for $\alpha = 0$ and $\alpha \to 1$. Therefore, the function takes non-negative values if and only if it has a zero, and $c_0$ is defined as the smallest value of $c$ for which this happens. Moreover, the function is strictly increasing in $c$ (cf. Remark~\ref{rem:functionanalysis}), so any larger value of $c$ will admit some value of $\alpha$ for which the function is strictly positive. This proves the statements for the \olea.

For the \folea, let us first consider the case $m_2/m_1 \leq 1-\delta$. In this case we may bound the second term of $\Phi$ by $0$ and obtain
\begin{align*}
\Phi \leq \frac{\E[s(s-1) (1-\alpha)^{s-1}]}{\E[s(1-\alpha)^{s-1}]} \stackrel{(*)}{\leq} \frac{\E[s(s-1)]}{\E[s]} = \frac{m_2}{m_1} \leq 1-\delta,
\end{align*}
where (*) follows from Chebyshev's sum inequality since the factor $(s-1)$ is increasing and $(1-\alpha)^{s-1}$ is decreasing in $s$. This settles the cases in which the \folea is successful. For the second part of the theorem, we have already shown that a distribution satisfying~\eqref{eq:hard2} for some $0<\alpha <1$ needs exponential time. It remains to show that~\eqref{eq:hard2} is satisfied for the special case listed in the theorem. Assume first that $m_2/m_1 \geq 1+\delta$, and that $\Pr[s=1] \leq 1/(Cs_0)$, where $C>0$ is a sufficiently large constant that we choose later. Here $s_0$ is as in the theorem, i.e., $m_{2,\leq s_0} = \sum_{\sigma=1}^{s_0} p_{\sigma}\sigma(\sigma-1) \geq (1+\delta/2)m_1$. 
Since the condition $m_2/m_1>1+\delta$ stays true if we make $\delta$ smaller, we may assume $\delta \leq 1/10$. Choose $\alpha := 1/(C' s_0)$, where $C' := 16/\delta$. 
Then $\alpha \leq 1/2$, which implies $1-\alpha \geq e^{-2\alpha}$. Hence, for all $\sigma \in [s_0]$ we have $(1-\alpha)^{\sigma} \geq (1-\alpha)^{s_0} \geq e^{-2\alpha s_0} = e^{-2/C'} \geq 1-2/C'$. Therefore,
\begin{align*}
\E[s(s-1)(1-\alpha)^{s-1}] & \geq \sum_{\sigma = 1}^{s_0} p_{\sigma}s(s-1)\left(1-\frac{2}{C'}\right) \\
& \geq \left(1+\frac{\delta}{2}\right)m_1\left(1- \frac{2}{C'}\right) \\
& \geq \left(1+\frac{\delta}{4}\right)m_1.
\end{align*}
Moreover, if we choose $C \geq 8C'/(\delta m_1)$, then we may bound
\begin{align*}
\frac{(1-\alpha)}{\alpha}\Pr[s=1] \leq C's_0\cdot \frac{1}{Cs_0} = \frac{C'}{C} \leq \frac{\delta}{8}m_1.
\end{align*}
Plugging this into $\Phi$, we get that
\begin{align*}
\Phi & = \frac{\E[s(s-1) (1-\alpha)^{s-1}]-\tfrac{(1-\alpha)}{\alpha}\Pr[s=1]}{\E[s(1-\alpha)^{s-1}]} \\ 
&\geq \frac{\left(1+\tfrac{\delta}{8}\right)m_1}{\E[s]} = 1+\frac{\delta}{8} 
\end{align*}
as required.

The second special case for the \folea is that $\mathcal D$ is a power law distribution with exponent $\kappa \in (1,2)$, i.e., $p_{\sigma} = \Theta(\sigma^{-\kappa})$.This case is similar as the previous case, since we have for all $s_0 \in N$,
\begin{align*}
m_{2,\leq s_0} = \sum_{\sigma = 1}^{s_0} p_{\sigma}\sigma(\sigma-1) = \sum_{\sigma = 1}^{s_0} \Theta(\sigma^{2-\kappa}) = \Omega(s_0^{3-\kappa}) = \omega(s_0),
\end{align*}
where the Landau notation in this case is with respect to $s_0 \to \infty$ instead of $n\to\infty$. For $\alpha := 1/s_0$ we have $1-\alpha \geq e^{-2\alpha}$ and thus $(1-\alpha)^{s_0} \geq e^{-2}$. Hence, if $s_0$ is a sufficiently large constant, 
\begin{align*}
\E[s(s-1)(1-\alpha)^{s-1}] & \geq e^{-2}m_{2,\leq s_0} \geq \frac{2}{\delta} s_0  \\
& \geq \frac{2}{\delta} \cdot \frac{(1-\alpha)}{\alpha}\Pr[s=1].
\end{align*}
Therefore,
\begin{align*}
\Phi & \geq \frac{(1-\tfrac{\delta}{2})\E[s(s-1) (1-\alpha)^{s-1}]}{\E[s(1-\alpha)^{s-1}]} \geq \left(1-\frac{\delta}{2}\right)\frac{m_2}{m_1} \geq 1+\frac{\delta}{4},
\end{align*}
where the last step holds for all $\delta \leq 1/2$.

The third special case for the \folea is that $p_1 \leq \tfrac49 \cdot p_3-\delta$. In this case, choose $\alpha = 1/3$ and observe that $3p_3(1-\alpha)^2 \geq p_1/\alpha +3\delta$. Using this, 
\begin{align*}
\sum_{\sigma = 1}^{3}& p_{\sigma}\sigma(\sigma-1)(1-\alpha)^{\sigma-1} - \frac{1-\alpha}{\alpha} p_1 \\
& \geq 2p_2(1-\alpha) + 3p_3(1-\alpha)^2 + \frac{p_1}{\alpha} +3\delta - \frac{1-\alpha}{\alpha} p_1 \\
& = \sum_{\sigma = 1}^{3} p_{\sigma}\sigma(1-\alpha)^{\sigma-1} + 3\delta.
\end{align*}
Hence,
\begin{align*}
\Phi & \geq \frac{\sum_{\sigma = 1}^{3}p_{\sigma}\sigma(1-\alpha)^{\sigma-1} + 3\delta + \sum_{\sigma = 4}^{\infty}p_{\sigma}\sigma(\sigma-1)(1-\alpha)^{\sigma-1}}{\E[s(1-\alpha)^{s-1}]} \\
& \geq \frac{\E[s(1-\alpha)^{s-1}] +3\delta}{\E[s(1-\alpha)^{s-1}]} = 1+\Omega(1).
\end{align*}
This settles the last case of the \folea.\medskip

So far we have analysed all cases with $\mu= 1$, so let us turn to $\mu >1$. We first give a general argument for the case that all individuals in the population have at most $\zeta n$ zero bits, for some sufficiently small constant $\zeta >0$ which may depend on the constants in the theorem (e.g., on $\mu$). In the following, we will use the Landau notation only to hide factors that are independent of $\zeta$. For example, $A = O(B)$ means that $B/A$ is bounded by some constant which is independent of $\zeta$ (but which may depend on the constants in the theorem). 

Assume that $x^{(1)}$ is a search point of maximal fitness in the current population, and let $y$ be the offspring in the current generation. We observe that $\Pr[y = x^{(1)}] = \Omega(1)$: for the GA's there is a constant probability that $y$ is generated by crossing $x^{(1)}$ with itself, since $\mu = O(1)$. For the EA's there is a constant probability to make a mutation without bit flips. (For the fEA's and fGA's this is an explicit condition.) Moreover, since $x^{(1)}$ has maximal fitness in the population, with constant probability both $x^{(1)}$ and $y$ survive the selection step. (They may be eliminated if the whole population has the same fitness.) By the same argument, there is a small, but $\Omega(1)$ probability that $x^{(1)}$ duplicates in $\mu$ successive rounds, and all its offspring survive all $\mu-1$ selection steps. Hence, with probability $\Omega(1)$ the population degenerates to $\mu$ copies of the same individuals. We say in this case that the search point and the round are \emph{consolidated}. Since this happens in each batch of $\mu$ rounds with constant probability, the expected time until some search point is consolidated is $O(1)$.

We first consider the \moea and \fmoea. To apply Theorem~\ref{thm:generic}, we will reinterpret the $(\mu+1)$ algorithms as follows. Assume that at some point in time $t_1$ there is a consolidated search point with at most $\zeta n/2$ zero bits. We call this search point $x^{(1)}$. Then we define recursively $t_i$ to be the minimal $t> t_{i-1}$ such that there is a consolidated search point at time $t$.
We define $x^{(i)}$ to be the consolidated search point at time $t_i$. In this way, the sequence of $x^{(i)}$ fits the description of an algorithm in Theorem~\ref{thm:generic}, although the process of going from $x^{(i-1)}$ to $x^{(i)}$ is rather complex. To complete the description, we still need to define the offspring $x'$ that appears in the algorithm description in Theorem~\ref{thm:generic}. We define it as the first offspring that is created from $x$. 
If $x'$ or $x$ are consolidated, then this fits the definition of the algorithm.\footnote{In fact, if $f(x)=f(x')$ then this does not quite fit the description of the algorithm, since we might consolidate $x$ while the elitist algorithm would always choose $x'$. However, the function \HT is symmetric with respect to any search points which have the same fitness, i.e., for any two such search points there is an automorphism of $\{0,1\}^d$ which interchanges the search points, but which leaves \HT invariant. Thus it does not matter which of the two search points we choose.} Otherwise we blame it to the adversary. Thus we need to show that the adversary is limited as required by the algorithm. Note that the distribution of $x'$ is just the distribution $\mathcal D$ of the mutation operator. Thus the same results as for the \olea and \folea immediately carry over if we can show that the adversary is sufficiently limited.

To estimate the effect of the adversary, assume that the current consolidated search point $x$ has at most $\zeta n$ zero bits, and consider the first mutant $x'$ with $f(x') > f(x)$. Afterwards, the population consists of $\mu-1$ copies of $x$ and one copy of $x'$. In each subsequent round, there are four (non-exclusive) possibilities:
\begin{enumerate}[(i)]
\item Another copy of $x$ is created. This happens with probability $\Omega(1)$.
\item Another copy of $x'$ is created. This happens with probability $\Omega(1)$.
\item A mutation $\neq x,x'$ is created with no flipped zero bit.
\item A mutation is created with at least one flipped zero bits. This happens with probability $O(\zeta)$.
\end{enumerate}
Note that until (iv) happens, all search points in the population are either equal to $x$ or $x'$, or are strictly dominated by $x$ or $x'$. In particular, all search points in the population are either copies of $x'$, or have a strictly worse fitness than $x'$. Therefore, $x'$ will be consolidated after an expected $O(1)$ number of steps, unless case (iv) occurs before that. Thus the probability that $x'$ is consolidated before case (iv) occurs is $1-O(\zeta)$. This means that the adversary may only act with probability $O(\zeta)$, which is sufficiently small if $\zeta$ is small. 

It remains to estimate the number $\tau$ of bits in which $x$ differs from the next consolidated search point $y$. Note that in any sequence of $\mu = O(1)$ rounds, we have probability $\Omega(1)$ that the population is consolidated, so the probability that we see at least $k$ rounds before consolidation drops exponentially in $k$. This implies that $\E[\tau] = O(1)$. Note that it would already imply exponentially falling tail bounds on $\tau$ for the \moea, but for the general case of the \fmoea we need to use a similar argument as for the \folea, as follows.

The next consolidated search point $y$ must satisfy $\HT(y) \geq \HT(x)$, since otherwise it could not supersede $x$ in the population. Let $i:= \ell(x)+1$ be the index of the current hot topic, and let $s$ be the number of bit flips to create $x'$. Then $\HT(x') \geq \HT(x)$ can only happen if either the level increases, or $d(A_i)$ does not increase. Therefore, $\Pr[\HT(x') \geq \HT(x) \mid s = \sigma] = e^{-\Omega(\sigma)}$, since the number of one-bits in $A_i$ \emph{increases} in expectation by $\alpha \sigma (1-2\zeta) = \Omega(\sigma)$, and likewise for the number of one-bits in $B_{i+1}$. Therefore, $\Pr[s \geq \sigma] = e^{-\Omega(\sigma)}$. Similarly, if $x''$ is the next offspring (either from $x$ or from $x'$), then the probability that $x''$ survives selection in this rouns is exponentially decreasing in the number $s''$ of bits flips, since $\Pr[\HT(x'') \geq \HT(x)] = e^{-\Omega(s'')}$. Repeating this argument, we see that for any fixed number of rounds the total number of bit flips in these rounds has an exponential tail bound. Since the number of rounds before consolidation has also an exponential tail bound, this proves the exponential tail bound on $\tau$. This concludes the proof for the \moea and \fmoea. 

Note for later use that the same tail bound argument also applies for the \moga and \fmoga because crossovers can only change bits that have been touched since the last consolidated round. Therefore, they do not increase the total number of bits that are touched between two consolidated rounds.\smallskip

For the \moga and \fmoga, note that the exponential runtime statements for small $\mu$ follow trivially from the \ooea and \fooea, since they agree with \moga and \fmoga if $\mu=1$. So let us consider the upper runtime bounds for large $\mu$. The situation is similar to the one for \moea and \fmoea, but with the crucial difference that the errors made in the creation of $x'$ may be repaired by crossovers between $x$ and $x'$. Other than before, we will apply part (c) of Theorem~\ref{thm:generic}.

Assume as before that $x$ is a consolidated search point. Note that crossovers cannot create new search points, so assume that that an offspring $x'\neq x$ with $\HT(x') > \HT(x)$ is created from $x$ by a mutation. Let $S_{01}$ and $S_{10}$ be the sets of bits that were flipped from zero to one and from one to zero, respectively, and let $s_{01} = |S_{01}|$, $s_{10} = |S_{10}|$, and $s = s_{01} +s _{10}$. Note that $s_{01} > 0$. As before, we have $\Pr[s \geq \sigma] = e^{-\Omega(\sigma)}$. Let $s_0$ be a constant such that $\Pr[s \geq \sigma] \leq \eta/4$, where $\eta$ is the constant from Theorem~\ref{thm:generic}. Note that by making $\zeta$ small enough, we can also bound the probability that an additional zero-bit is flipped before the next consolidated round by $\eta/4$. So let us assume that $s \leq \sigma$ and that no additional zero-bits are flipped until consolidation. 

Let $k$ be the number of search points in the population that are not copies of $x$. In rounds where the parents of mutation or crossover are picked among the copies of $x$, the population does not change. Otherwise, i.e., when at least one parent is not $x$, there is a probability of at least $1/k$ that it is a copy of $y$. With probability at least $1/2$ the operation in this round is a crossover (crossover has a larger probability than mutation since two parents are picked). Moreover, if $k\leq \mu/2$, then with probability at least $1/2$ it is a crossover with $x$. Therefore, the expected number of crossover children between $x$ and $x'$ is at least $\sum_{k=1}^{\mu/2} 1/(4k) \geq \tfrac14 \ln(\mu/2)$. Note that this becomes arbitrarily large if $\mu$ is large. In particular, for sufficiently large $\mu$, with probability at least $1-\eta/4$ there will be at least $C_2$ crossovers between $x$ and $x'$, for any constant $C_2$ that we desire. 

Now observe that since $s\leq \sigma_0$, every crossover copy between $x$ and $x'$ has probability at least $2^{-\sigma_0} = \Omega(1)$ to retain the bits in $S_{01}$ from $x'$ and to pick all bits in $S_{10}$ from $x$. In particular, if $\mu$ is large enough, then with probability $1-\eta/4$ this happens at least once. If it happens, then the offspring $y$ dominates $x$, $x'$, and any crossover of $x$, and $x$. Moreover, since we assume that no zero-bits are flipped by mutations, it also dominates any mutation offspring of any search point in the population. Therefore, the population must consolidate with $y$. In this case we say that $x'$ was \emph{fully repaired}. Note that whenever an offspring $x'$ of a consolidated search point $x$ with at most $\zeta n$ zero bits is created with $\HT(x') > \HT(x)$, then it has probability at least $1-\eta$ to be fully repaired. 

We are now ready to explain how we apply Theorem~\ref{thm:generic} (c). As before, let $x^{(1)}$ be the first consolidated search point. Then we define recursively $t_i$ to be the minimal $t> t_{i-1}$ such that there is a consolidated search point at time $t$, \emph{and} such that at least one mutation happens in rounds $t_{i-1},\ldots,t$. The latter condition simply means that we ignore rounds in which a consolidated search point performs a crossover with itself. We define $x^{(i)}$ to be the consolidated search point at time $t_i$, and we define $x'^{(i)}$ to be the first mutation offspring after time $t_{i}$. If $x$ has more than $\zeta n$ zero bits then we define the winner offspring $y^{(i)}$ to be $x'^{(i)}$, which gives an unbiased distribution. If $x$ has at most $\zeta n$ zero-bits, then we define $y^{(i)} := x^{(i)}$ if $\HT(x'^{(i)}) \leq \HT(x^{(i)})$, and we define $y^{(i)}$ to be the fully repaired $x'^{(i)}$ if $\HT(x'^{(i)}) \leq \HT(x^{(i)})$. In the latter case, we have shown that indeed $x^{(i+1)} = y^{(i)}$ with probability at least $1-\eta$, so we may blame any other outcome to the adversary. The power of the adversary is limited in the same way as for the \moga and \fmoga, so we may indeed apply Theorem~\ref{thm:generic} (c). This concludes the proof.
\end{proof}

\section{Conclusions}
\label{sec:conclusions}
We have studied a large set of algorithms, and we have shown that in all cases without crossover, there is a dichotomy with respect to a parameter ($c$, $c\gamma$, or $\Phi$, where the latter one is related to $m_2/m_1$) for optimising the monotone function family \hottopic. If the parameter is small, then the algorithms need time $O(n\log n)$; if the parameter is large then they need exponential time on some instances. In the cases \olea, \fooea \ollga, and for good start points also \folea, if the parameter is small, then we could show that the algorithms are actually fast on \emph{all} monotone functions. However, there are many open problems left, and we conclude the paper by a selection of those.
\begin{itemize}
\item We have analysed the algorithms theoretically for the case $n\to \infty$. We have only provided a very modest number of experimental data points for the \ooea, as a proof of concept to show that the dichotomy can be clearly observed in data. However, more experiments are sorely needed to understand for what values of $n$ the effects become observable. For example, do larger values of $\lambda$ and $\mu$ \emph{delay} the detrimental effects of \hottopic, so that it is only visible for larger $n$?
\item In some cases our runtime bounds for small parameter values hold only for \hottopic, but the general status of monotone functions remains unclear (\moea,\fmoea). So does a small mutation parameter guarantee a small runtime on \emph{all} monotone functions?
\item We could show that genetic algorithms are superior to evolutionary algorithms on the \hottopic functions. However, is the same true in general for monotone functions? Is it true that the \moga and \fmoga are fast for all monotone functions if $\mu$ is large enough?
\item It seems important to understand more precisely how large $\mu$ should be in GA's to cope with larger mutation parameters. For example, for the \moga with mutation parameter $c$, how large does $m$ need to be so that it is still fast on all \hottopic instances?
\item By now a classical question is: are there monotone functions which are hard for the parameter range $[1,2.13..)$? Most intriguingly: are there hard monotone instances for the \ooea for every $c >1$? For $c=1$ it is known that the runtime is polynomial, but is it always $O(n\log n)$?
\item Our proofs for population sizes $\mu >1$ rely on the fact that in all considered algorithms diversity tends to be lost close to the optimum. Do the results stay the same if diversity is actively maintained, for example by duplication avoidance or by genotypical or phenotypical niching?
\item How is the performance of algorithms that change the mutation strength dynamically, e.g., with the $1/5$-th rule? In the introduction we have given an intuition why this might be bad, but intuition has failed before on monotone functions.
\item While \hottopic is defined in a discrete setting, the underlying intuition is related to continuous optimisation. Is there a continuous analogue of \hottopic, and what is the performance of optimisation algorithms like the CMA-ES or particle swarm optimisation?
\end{itemize}

\subsection*{Acknowledgments}
Part of the work was inspired by discussions at the Dagstuhl meeting 19171 on Theory of Randomized Optimization Heuristics. In particular, we thank Benjamin Doerr for proposing to consider the \olea on monotone functions, which started this line of research. The study of monotone functions was fostered by the COST Action CA15140 ``Improving Applicability of Nature-Inspired Optimisation by Joining Theory and Practice'', and the author has proposed to include monotone functions into the set of benchmarks developed by working group 3.

}


\end{document}